\pdfoutput=1
\documentclass{article} 
\usepackage{iclr2023_conference,times}

\usepackage{amsmath,amsfonts,bm}

\def\1{\bm{1}}

\DeclareMathAlphabet{\mathsfit}{\encodingdefault}{\sfdefault}{m}{sl}
\SetMathAlphabet{\mathsfit}{bold}{\encodingdefault}{\sfdefault}{bx}{n}

\usepackage{hyperref}
\usepackage{url}
\usepackage{mathtools}
\usepackage{tikz}
\usepackage{wrapfig}
\usepackage{amsthm}
\usepackage{enumitem}
\usepackage{amsfonts}
\usepackage{amssymb}
\usepackage{amsmath}
\usepackage{graphicx}
\usepackage{setspace}
\usepackage{multirow}
\usepackage{multicol}
\usepackage{booktabs}
\usepackage{adjustbox}
\usepackage{textcomp}
\usepackage{array}
\usepackage[bb=dsserif]{mathalpha}
\usepackage{bm}
\usepackage[T1]{fontenc}
\usepackage{color, colortbl}
\definecolor{Gray}{gray}{0.9}

\newcommand{\splitcell}[1]{\begin{tabular}{@{}c@{}}#1\end{tabular}}
\newcommand{\bsplitcell}[1]{$\left[\splitcell{#1}\right]$}
\newcommand*\circled[1]{\tikz[baseline=(char.base)]{
            \node[shape=circle,draw,inner sep=0.5pt] (char) {#1};}}
\makeatletter
\DeclareFontFamily{U}{tipa}{}
\DeclareFontShape{U}{tipa}{m}{n}{<->tipa10}{}
\newcommand{\arc@char}{{\usefont{U}{tipa}{m}{n}\symbol{62}}}%
\newcommand{\arc}[1]{\mathpalette\arc@arc{#1}}
\newcommand{\arc@arc}[2]{%
  \sbox0{$\m@th#1#2$}%
  \vbox{
    \hbox{\resizebox{\wd0}{\height}{\arc@char}}
    \nointerlineskip
    \box0
  }%
}
\makeatother

\newcommand{\model}{EurNet}

\newtheorem{myprop}{Proposition}

\setlist[itemize]{leftmargin=5.4mm}
\setlist[enumerate]{leftmargin=8.4mm}

\definecolor{revision}{RGB}{255,0,0}            

\title{{\model}: Efficient Multi-Range Relational \\Modeling of Spatial Multi-Relational Data}

\author{
  Minghao Xu\textsuperscript{1,4}, Yuanfan Guo\textsuperscript{3}, Yi Xu\textsuperscript{3}, \\ \textbf{Jian Tang\textsuperscript{1,5,6}, Xinlei Chen\textsuperscript{2}, Yuandong Tian\textsuperscript{2}} \\
  Mila - Qu\'ebec AI Institute\textsuperscript{1}, Meta AI (FAIR)\textsuperscript{2}, Shanghai Jiao Tong University\textsuperscript{3} \\
  Universit\'e de Montr\'eal\textsuperscript{4}, HEC Montr\'eal\textsuperscript{5}, CIFAR AI Chair\textsuperscript{6} \\
  \texttt{minghao.xu@mila.quebec}, 
  \texttt{\{gyfastas,xuyi\}@sjtu.edu.cn} \\
  \texttt{jian.tang@hec.ca},
  \texttt{\{xinleic,yuandong\}@meta.com}
}

\begin{document}

\maketitle



\begin{abstract}

Modeling spatial relationship in the data remains critical across many different tasks, such as image classification, semantic segmentation and protein structure understanding. Previous works often use a unified solution like relative positional encoding. However, there exists different kinds of spatial relations, including short-range, medium-range and long-range relations, and modeling them separately can better capture the focus of different tasks on the multi-range relations (\emph{e.g.}, short-range relations can be important in instance segmentation, while long-range relations should be upweighted for semantic segmentation).  
In this work, we introduce the \textbf{\model} for \textbf{E}fficient m\textbf{u}lti-range \textbf{r}elational modeling. {\model} constructs the multi-relational graph, where each type of edge corresponds to short-, medium- or long-range spatial interactions. 
In the constructed graph, {\model} adopts a novel modeling layer, called \emph{\textbf{g}ated \textbf{r}elational \textbf{m}essage \textbf{p}assing} (\textbf{GRMP}), to propagate multi-relational information across the data. GRMP captures multiple relations within the data with little extra computational cost. We study {\model}s in two important domains for image and protein structure modeling. Extensive experiments on ImageNet classification, COCO object detection and ADE20K semantic segmentation verify the gains of {\model} over the previous SoTA FocalNet. On the EC and GO protein function prediction benchmarks, {\model} consistently surpasses the previous SoTA GearNet. Our results demonstrate the strength of {\model}s on modeling spatial multi-relational data from various domains. The implementations of {\model} for image modeling are available at \url{https://github.com/hirl-team/EurNet-Image}. The implementations for other applied domains/tasks will be released soon.  

\end{abstract}


\section{Introduction} \label{sec:intro}


This work studies the data that lie in the 2D/3D space and incorporate interacting relations on different spatial ranges. 
 A representative example is the image data, where an object in the image can interact with other adjacent objects via the direct touch, and it can also interact with those distantly relevant ones via gazing, waving hands or pointing. In protein science, the protein 3D structure is another typical example, in which different amino acids can interact in short range by peptide/hydrogen bonds, and they can also interact in medium and long ranges by hydrophobic interaction. We summarize such kind of data as \textbf{spatial multi-relational data}. 

In various domains, a lot of previous efforts have been made to model the spatial multi-relational data. For image modeling, multi-head self-attention mechanisms~\citep{dosovitskiy2020image,liu2021swin}, convolutional operations with large receptive fields~\citep{ding2022scaling,yang2022focal} and MLPs for mixing full spatial information~\citep{tolstikhin2021mlp,touvron2021resmlp} are explored to capture multi-range spatial interactions within an image. For protein structure modeling, \citet{zhang2022protein} builds multiple groups of edges for different short-range interactions and employs relational graph convolution~\citep{schlichtkrull2018modeling} for multi-relational modeling. These works either implicitly treat different kinds of spatial relations (\emph{i.e.}, short-range, medium-range and long-range relations)~\citep{tolstikhin2021mlp,yang2022focal} or handle them by a unified scheme like relative positional encoding~\citep{dosovitskiy2020image,liu2021swin}. However, considering the relative importance of these spatial relations could vary across different tasks (\emph{e.g.}, the great importance of short-range relations in instance segmentation, and the upgraded importance of long-range relations in semantic segmentation), separately modeling each spatial relation is a better solution to capture different tasks' focus. Such a separate modeling approach remains to be explored, and, especially, the approach is expected to have efficient adaptation to large data and model scales.

To attain this goal, we propose the \textbf{\model} for \textbf{E}fficient m\textbf{u}lti-range \textbf{r}elational modeling. In general, {\model}s are a series of relational graph neural networks equipped with graph construction layers, where relational edges are constructed by the layers for capturing multi-range spatial interactions. When instantiated with different domain knowledge (\emph{e.g.}, computer vision or protein science), {\model}s can be specialized to tackle important problems like image classification, image segmentation and protein function prediction. To be specific, upon the raw data, {\model} first uses the graph construction layers to build different types of edges that respectively capture the short-, medium- and long-range spatial interactions within the data.
For efficient multi-relational modeling over the constructed graph, we next introduce the \emph{gated relational message passing} (\textbf{GRMP}) layer as the basic modeling module of {\model}. GRMP separately performs (1) relational message aggregation on each individual feature channel and (2) node-wise aggregation of different feature channels. 
Compared to the classical relational graph convolution (RGConv)~\citep{schlichtkrull2018modeling}, GRMP enjoys lower computational cost when more relations are to be modeled, and thus can handle more types of spatial interactions given the same computational budget. {\model} also supports dynamic graph construction and multi-stage modeling that are used in domains like image modeling. 

We demonstrate {\model}s in image and protein structure modeling. To model image patches, we build {\model}s with hierarchical graph construction layers and multiple modeling stages and derive a model series with increasing capacity, \emph{i.e.}, \textbf{{\model}-T}, \textbf{{\model}-S} and \textbf{{\model}-B}. These models enjoy comparable or better top-1 accuracy (82.3\% \emph{v.s.} 82.3\%; 83.6\% \emph{v.s.} 83.5\%; 84.1\% \emph{v.s.} 83.9\%) against the previous SoTA FocalNet$_{\textrm{(LRF)}}$ series~\citep{yang2022focal} on ImageNet-1K classification. 
To model protein alpha carbons, we build {\model} with a single-stage model architecture as GearNet~\cite{zhang2022protein}. 
On standard protein function prediction benchmarks, {\model} consistently outperforms the SoTA GearNet in terms of $\mathbf{F}_\mathbf{max}$ score (EC: 0.768 \emph{v.s.} 0.730; GO-BP: 0.437 \emph{v.s.} 0.356; GO-MF: 0.563 \emph{v.s.} 0.503; GO-CC: 0.421 \emph{v.s.} 0.414). 
Our results demonstrate that {\model} could be a strong candidate for modeling spatial multi-relational data in various domains. 


\section{Related Work} \label{sec:related}



\textbf{Multi-relational data modeling.} Multi-relational data are ubiquitous in the real world, \emph{e.g.}, knowledge graphs~\citep{toutanova2015observed} and customer-product networks~\citep{li2014recommendation}. To effectively model multiple types of relations/interactions, existing works have explored embedding-based methods~\citep{bordes2013translating,sun2019rotate}
and different relational graph neural networks (GNNs)~\citep{schlichtkrull2018modeling,vashishth2019composition,zhu2021neural}. Previous relational GNNs mainly focus on model expressivity, and few works~\citep{li2021dimensionwise} study the computational efficiency for relational modeling at scale. In addition, they can hardly model the spatial multi-relational data whose relational linking structures at different spatial ranges are not originally given. 
{\model} is designed to model such kind of data in a computationally efficient way. 

\textbf{Image modeling.} After the dominance of convolutional vision backbones~\citep{he2016deep,tan2019efficientnet} in 2010s, researchers rethink the architectures for more effective image modeling in 2020s. Vision Transformers~\citep{dosovitskiy2020image,liu2021swin,wang2021pyramid} replace convolutions with the self-attention mechanism~\citep{vaswani2017attention} to better capture non-local interactions and gain SoTA performance. Following such successes, modern convolutional architectures~\citep{liu2022convnet,yang2022focal}, all-MLP architectures~\citep{tolstikhin2021mlp,touvron2021resmlp} and vision GNNs~\citep{han2022vision} are designed to aggregate long-range spatial context. Some earlier works~\citep{chen2019graph,zhang2019dual,zhang2020dynamic} realize non-local modeling by graph convolution on fully-connected or dynamic graphs.
By comparison, {\model} captures multi-range spatial interactions from a novel graph learning perspective, \emph{i.e.}, multi-relational modeling. 


\textbf{Protein structure modeling.} 
Diverse geometric encoders are designed to model different levels of protein structures, 
including residue-level structures~\citep{gligorijevic2021structure,zhang2022protein}, atom-level structures~\citep{jing2021equivariant} and protein surfaces~\citep{gainza2020deciphering,sverrisson2021fast}. 
GearNet~\citep{zhang2022protein} models multi-relational short-range residue interactions with relational graph convolution (RGConv).
By comparison, our {\model} models a broader range of residue interactions including short, medium and long ranges, and it studies the gated relational message passing (GRMP) as a more efficient and equally effective alternative of RGConv. 


\vspace{-0.1mm}
\section{{\model} for Efficient Multi-Range Relational Modeling} \label{sec:method}
\vspace{-0.1mm}


\vspace{-0.1mm}
\subsection{Problem Definition} \label{sec:method:problem}
\vspace{-0.1mm}

This work studies the data $\mathcal{V} = \{v_i\}_{i=1}^N$ with $N$ data units (\emph{e.g.}, patches in an image, alpha carbons in a protein, \emph{etc.}) with the following structure: (1) \textbf{spatial interaction on multiple ranges}: data units can interact with each other across diverse spatial ranges; (2) \textbf{multi-relational interaction}: multiple interaction types (\emph{i.e.}, relations) exist between different units; (3) \textbf{no canonical linking structure}: the linking structures of multi-range interactions are not specified in the raw data. 


To effectively model such \emph{spatial multi-relational data}, the model is expected to own following capabilities: (1)~\textbf{dynamic multi-range linking}: the model can link relevant data units across different spatial ranges, and the linking structure can change along the whole model if desired; (2)~\textbf{multi-relational linking}: the model divides all links into multiple groups based on their interaction types; (3)~\textbf{efficient multi-relational modeling}: the model can propagate information among interacting units by taking their interaction types into consideration, and it will not introduce too much extra computation when involving more relations. Keeping all these requirements in mind, we next introduce the high-level designs of {\model}, and we present its detailed instantiations in Sec.~\ref{sec:app}.


\vspace{-0.3mm}
\subsection{Multi-Range Relational Graph Construction} \label{sec:method:graph}

We regard each data unit $v \in \mathcal{V}$ as a node in the graph. 
For the lack of canonical linking structure among the nodes,
we therefore seek to build edges among them, especially with considering their interactions on multiple spatial ranges and dynamically adjusting the graph structure if desired. 

\textbf{Multi-range relational edge construction.} Given the concepts of spatial and semantic adjacency in a specific domain (\emph{e.g.}, computer vision or protein science), we construct three groups of edges {\small $\mathcal{E}_{\mathrm{short}} = \{ \{ (u,v,r) \} | r \in \mathcal{R}_{\mathrm{short}} \}$}, {\small $\mathcal{E}_{\mathrm{medium}} = \{ \{ (u,v,r) \} | r \in \mathcal{R}_{\mathrm{medium}} \}$} and {\small $\mathcal{E}_{\mathrm{long}} = \{ \{ (u,v,r) \} | r \in \mathcal{R}_{\mathrm{long}} \}$} to represent short-, medium- and long-range spatial interactions, where $(u,v,r)$ denotes an edge from node $u$ to node $v$ with relation $r$, and $\mathcal{R}_{\mathrm{short}}$/$\mathcal{R}_{\mathrm{medium}}$/$\mathcal{R}_{\mathrm{long}}$ is the set of relations for short-/medium-/long-range interactions. To capture the interactions on different spatial ranges, all these edges are gathered into the edge set {\small $\mathcal{E} = \mathcal{E}_{\mathrm{short}} \cup \mathcal{E}_{\mathrm{medium}} \cup \mathcal{E}_{\mathrm{long}} = \{ \{ (u,v,r) \} | r \in \mathcal{R} \}$} with the integrated relation set $\mathcal{R} = \mathcal{R}_{\mathrm{short}} \cup \mathcal{R}_{\mathrm{medium}} \cup \mathcal{R}_{\mathrm{long}}$. Now, the raw data $\mathcal{V}$ is structured as a multi-relational graph $\mathcal{G} = (\mathcal{V}, \mathcal{E}, \mathcal{R})$ that is aware of diverse types of interactions within the data. 

\textbf{Dynamic edge construction.} A model can focus on different levels of semantics at different modeling stages. For example, for the image modeling problem we consider, a typical hierarchical image encoder~\citep{he2016deep,liu2021swin} is split into multiple stages, and it tends to encode low-level features in shallower stages and encode high-level semantics in deeper stages. To accommodate such a hierarchical modeling manner, our graph construction scheme will be dynamically performed before each modeling stage based on the input features (\emph{e.g.}, node coordinates or representations) of the stage, so that each modeling stage can explore its specific neighborhood structures of data units. 


\vspace{-0.3mm}
\subsection{Gated Relational Message Passing} \label{sec:method:layer}

To perform multi-relational modeling over the constructed graph $\mathcal{G}$, the typical method Relational Graph Convolution (RGConv)~\citep{schlichtkrull2018modeling} employs a unique convolutional kernel matrix $W_r$ to aggregate the messages of relation $r$, leading to $|\mathcal{R}|$ different kernel matrices in total for the message aggregation from neighborhoods. Taking node $v$ as an example, the RGConv layer updates its representation from $z_v$ to $z'_v$ as below:
\begin{equation} \label{eq:rgconv}
z^{\mathrm{aggr}}_v = \sum_{r \in \mathcal{R}} \sum_{u \in \mathcal{N}_r(v)} \frac{1}{|\mathcal{N}_r(v)|} \!\; W_r z_u, \quad z'_v = W^{\mathrm{self}} z_v + z^{\mathrm{aggr}}_v ,
\end{equation}
where $z^{\mathrm{aggr}}_v$ is the aggregated message for node $v$, $\mathcal{N}_r(v) = \{u | (u,v,r) \in \mathcal{E} \}$ are $v$'s neighbors with relation $r$, and $W^{\mathrm{self}}$ is the weight matrix for self-update (we omit all bias terms for brevity). 

We assume that, when introducing a new relation, the in-degree of each node will increase by $\bar{d}$ on average. By taking the efficient implementation of RGConv with sparse matrix multiplication, it can be shown that the floating-point operations (FLOPs) of RGConv with $C$-dimensional input and output node features has the following form (see Appendix~\ref{supp:sec:flops} for proof):
\vspace{0.5mm}
\begin{equation} \label{eq:rgconv_flops}
\mathrm{FLOPs}(\mathrm{RGConv}) = |\mathcal{R}| \cdot (2 \bar{d} |\mathcal{V}| C + 2 |\mathcal{V}| C^2) + 2 |\mathcal{V}| C^2 + |\mathcal{V}| C .
\vspace{0.6mm}
\end{equation}
Therefore, the computational cost will scale with the relation number $|\mathcal{R}|$ by the factor of $2 \bar{d} |\mathcal{V}| C + 2 |\mathcal{V}| C^2$. Considering both the node number $|\mathcal{V}|$ and the feature dimension $C$ could be large in many applications, the $2 |\mathcal{V}| C^2$ term will be the main obstacle of exploring more relations with moderate extra computation, which hurts the model capacity under a strict constraint on computational cost.

For more efficient multi-relational modeling, we aim at an approach that (1) can effectively model the interactions among relational messages and among feature channels, and (2) owns a gentle scaling behavior when modeling increasing number of relations within the data. To attain this goal, we propose the Gated Relational Message Passing (GRMP). 
Inspired by light-weight separable graph convolution methods~\citep{balcilar2020spectral,li2021dimensionwise} that aggregate neighborhood features in a channel-wise way, GRMP decomposes the relation-channel entangled aggregation of RGConv into (\emph{i}) the aggregation of intra- and inter-relation messages on each individual channel and (\emph{ii}) the aggregation of different feature channels. 
Specifically, it consecutively performs following steps: {\small \circled{1}} a pre-layer node-wise channel aggregation with the weight matrix $W^{\mathrm{in}}$, {\small \circled{2}} an intra-relation message aggregation through channel-wise graph convolution, {\small \circled{3}} an inter-relation message aggregation by node-adaptive weighted summation, {\small \circled{4}} a post-layer node-wise channel aggregation with the weight matrix $W^{\mathrm{out}}$, and {\small \circled{5}} the final node representation update by \emph{regarding the aggregated neighborhood information as gate}. Formally, GRMP updates the representation of node $v$ from $z_v$ to $z'_v$ as below:
\begin{equation} \label{eq:grmp}
z^{\mathrm{aggr}}_v = \overbrace{ W^{\mathrm{out}} \bigg( \overbrace{\sum_{r \in \mathcal{R}} \alpha_r(v) \cdot \underbrace{\sum_{u \in \mathcal{N}_r(v)} \frac{1}{|\mathcal{N}_r(v)|} \!\; w_r \odot ( \underbrace{W^{\mathrm{in}} z_u}_\text{step {\tiny \circled{1}}} )}_\text{step {\tiny \circled{2}}} }^\text{step {\tiny \circled{3}}} \bigg) }^\text{step {\tiny \circled{4}}} , 
\quad z'_v = \underbrace{ W^{\mathrm{self}} z_v \odot z^{\mathrm{aggr}}_v }_\text{step {\tiny \circled{5}}} ,
\end{equation}
where $\alpha(v) = W^{\alpha} z_v \in \mathbb{R}^{|\mathcal{R}|}$ are the attentive weights assigned to all relations on node $v$ ($W^{\alpha}$ is the weight matrix for node-adaptive relation weighting), $w_r$ is the channel-wise convolutional kernel vector for relation $r$ (with the same shape as the node feature vector after step {\small \circled{1}}), and $\odot$ denotes the Hadamard product. The definitions of $z^{\mathrm{aggr}}_v$, $\mathcal{N}_r(v)$ and $W^{\mathrm{self}}$ follow Eq.~\eqref{eq:rgconv}, and all biases are omitted. We analyze the key components of GRMP in Appendix~\ref{supp:sec:ablation:grmp}. We also provide a graphical illustration of the GRMP layer in Appendix~\ref{supp:sec:grmp_fig}.

Under the efficient implementation with sparse matrix multiplication, GRMP consumes the FLOPs as below when taking $C$-dimensional input and output node features (see Appendix~\ref{supp:sec:flops} for proof): 
\vspace{0.5mm}
\begin{equation} \label{eq:grmp_flops}
\mathrm{FLOPs}(\mathrm{GRMP}) = |\mathcal{R}| \cdot (2 \bar{d} + 7) |\mathcal{V}| C  + 6|\mathcal{V}| C^2 .
\vspace{0.6mm}
\end{equation}
\vspace{-5.3mm}

\begin{wrapfigure}{R}{0.36\textwidth}
\vspace{-5mm}
\centering
\includegraphics[width=0.36\textwidth]{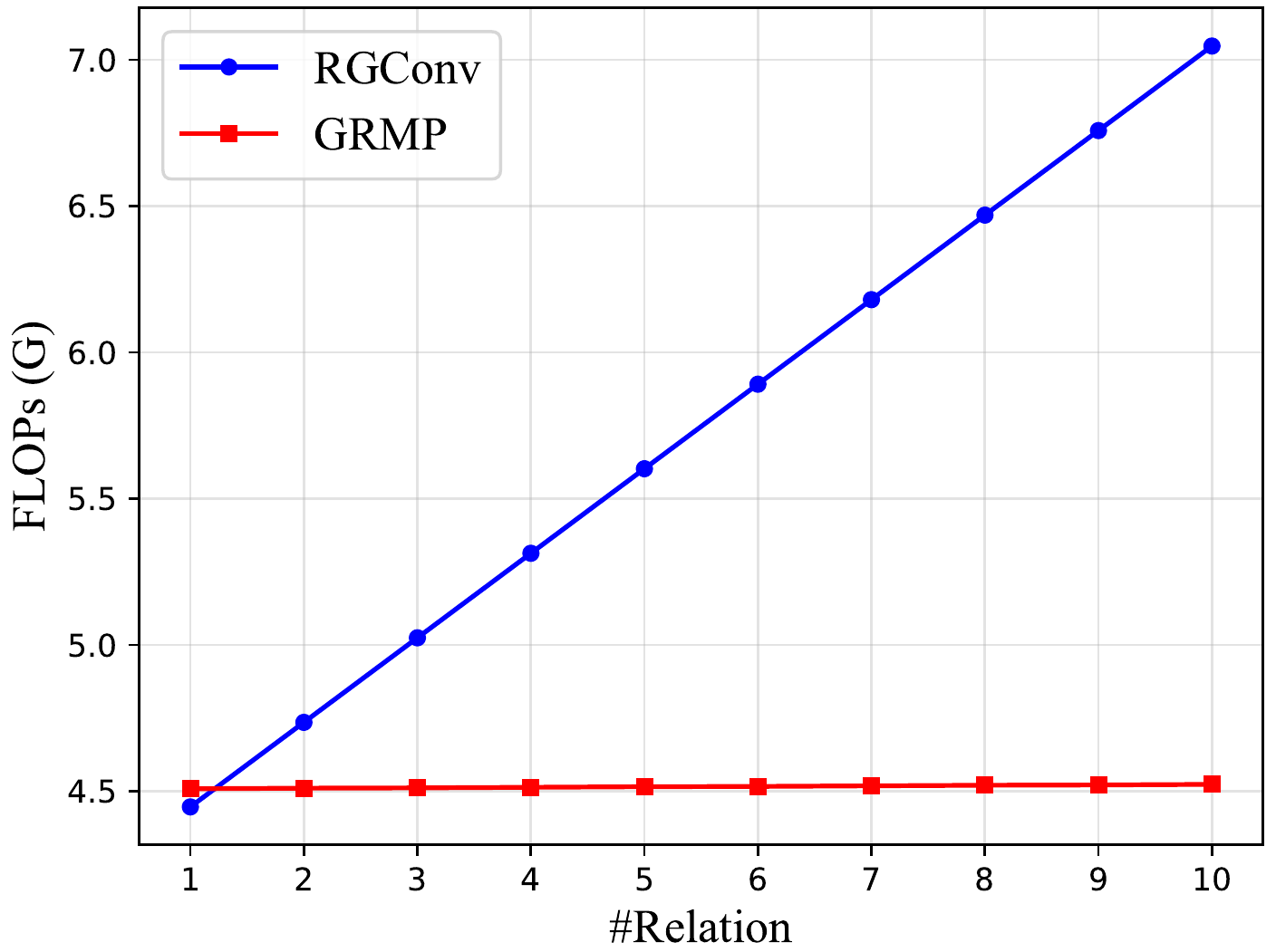}
\vspace{-8.0mm}
\caption{FLOPs trend of RGConv and GRMP under different relation numbers, evaluated on {\model}-T for image modeling.} 
\label{fig:flops_compare}
\vspace{-6mm}
\end{wrapfigure}
Therefore, the relation number $|\mathcal{R}|$ scales the computational cost of GRMP with the scaling factor $(2 \bar{d} + 7) |\mathcal{V}| C$. Compared to the scaling factor $2 \bar{d} |\mathcal{V}| C + 2 |\mathcal{V}| C^2$ of RGConv, this factor gets rid of the quadratic reliance on feature dimension and thus leads to a gentler scaling behavior when increasing the number of considered relations. In Fig.~\ref{fig:flops_compare}, we compare the FLOPs of RGConv and GRMP when they respectively serve as the building block of {\model}-T for image modeling (image resolution: $224 \times 224$; ``T'' denotes the tiny-scale model). In this illustrative comparison, we simply connect each node (\emph{i.e.}, image patch) with its $K$-nearest neighbors in terms of representation similarity, and the connection with the $k$-th nearest neighbor is regarded as the $k$-th relation, leading to $K$ relations in total. We can observe that, when increasing the number of neighbors and thus the number of relations, the computational cost of GRMP-based model increases much more gently than the RGConv-based one.
This merit enhances the efficiency and effectiveness of GRMP-based models in real-world problems like image and protein structure modeling, as studied in the second paragraph of Sec.~\ref{sec:exp:ablation} and in the Appendix~\ref{supp:sec:ablation:layer}.


\section{Instantiations of {\model}} \label{sec:app}
\vspace{-0.8mm}

In the main paper, we focus on two application domains, \emph{i.e.}, computer vision and protein science, where modeling spatial multi-relational data (\emph{i.e.}, images and protein structures) can solve important problems. In Appendix~\ref{supp:sec:KG}, we further study the effectiveness of {\model} on modeling an important kind of multi-relational data without spatial information, \emph{i.e.}, knowledge graphs.


\vspace{-0.6mm}
\subsection{{\model} for Image Modeling} \label{sec:app:cv}

\vspace{-0.7mm}
\subsubsection{Relational Edges for Short-, Medium- and Long-range Interactions} \label{sec:app:cv:edge}
\vspace{-0.3mm}

Following previous practices~\citep{dosovitskiy2020image,liu2021swin}, we split an image into local patches and regard these patches as the node set $\mathcal{V}$ of our multi-relational graph. Upon these patches, we construct following relational edges to capture different ranges of spatial interactions within an image (see Fig.~\ref{fig:image_edge} for a graphical illustration):
\vspace{-4.5mm}
\begin{itemize}
    \parbox[t]{\dimexpr\textwidth-\leftmargin}{
        \begin{wrapfigure}{R}{0.5\textwidth}
        \vspace{-5.1mm}
        \centering
        \includegraphics[width=0.5\textwidth]{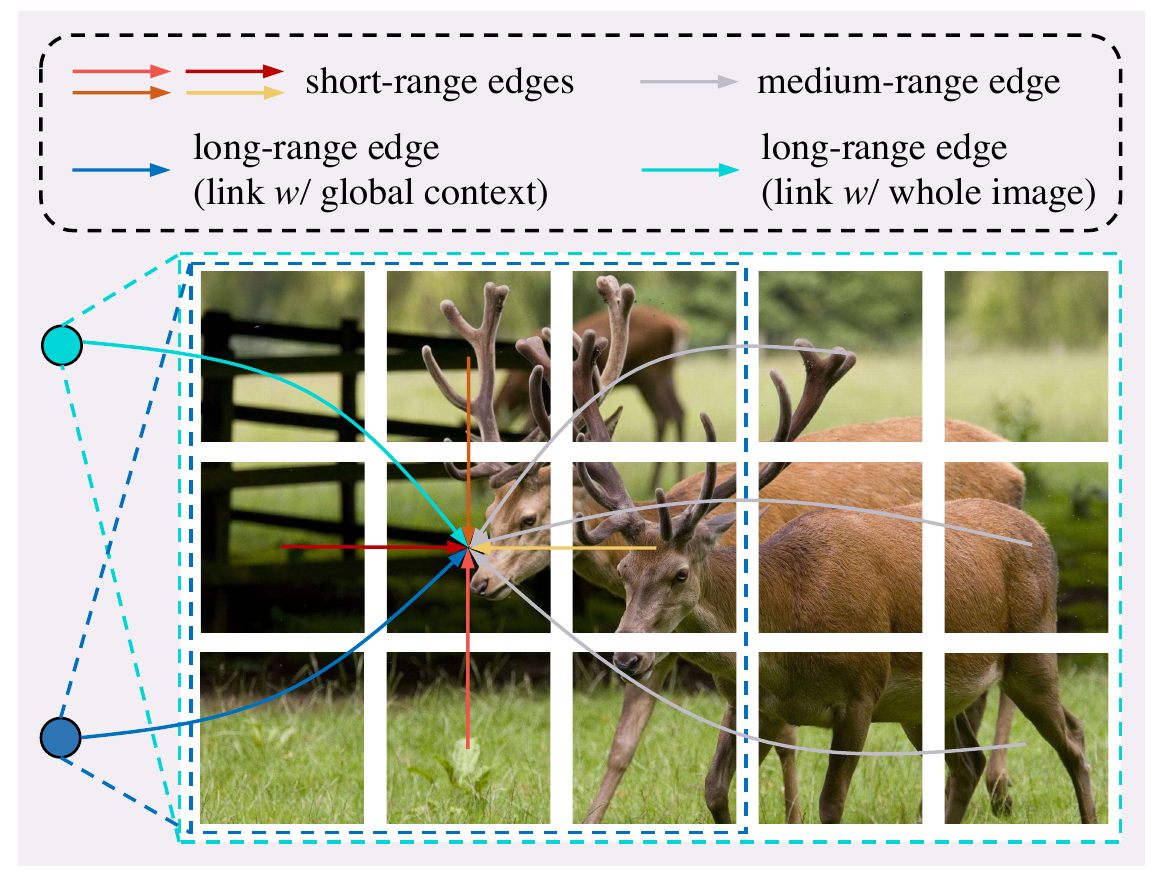}
        \vspace{-7.5mm}
        \caption{Multi-range relational edges for image.} 
        \label{fig:image_edge}
        \vspace{-5.1mm}
        \end{wrapfigure}
    \item \textbf{Edges for short-range interactions} ($|\mathcal{R}_{\mathrm{short}}| = 4$). We connect each patch with its up, down, left and right patches and regard each direction of adjacency as a relation. These edges capture the one-hop spatial neighbors and thus shortest-range spatial interactions of each image patch. 
    \vspace{0.5mm}
    \item \textbf{Edges for medium-range interactions} ($|\mathcal{R}_{\mathrm{medium}}| = 1$). 
    In the medium range, a patch can interact with other patches sharing similar semantics (\emph{e.g.}, different body parts of deer in Fig.~\ref{fig:image_edge}). We thus connect each patch with its $K$-nearest neighbors in terms of representation similarity measured by negative Euclidean distance (we analyze the sensitivity of $K$ in Appendix~\ref{supp:sec:sensitivity}), 
    and these edges are with the same relation. All edges connecting two patches within the same 2\texttimes2 window are removed to avoid short-range linking. 
    }
    \vspace{-0.5mm}
    \item \textbf{Edges for long-range interactions} ($|\mathcal{R}_{\mathrm{long}}| = 2$). To model long-range interactions, we introduce two kinds of \emph{virtual nodes} and the associated edges. (1) \emph{A virtual node for whole-image representation} is derived by global average pooling over all patch representations, and this virtual node is linked to all patches. (2) \emph{Per-patch virtual nodes for surrounding global context} are got by a stack of depth-wise 2D convolutions~\citep{yang2022focal} that aggregate each patch's contextual information with large receptive field and low computation; an edge links each of these virtual nodes to its corresponding patch with a different long-range interacting relation against that in (1), due to the different global context levels represented by two kinds of virtual nodes. 
\end{itemize}
\vspace{-1.3mm}

By gathering all these edges representing 7 different relations, we have the edge set $\mathcal{E}$, the relation (\emph{i.e.}, edge type) set $\mathcal{R}$ and the full graph $\mathcal{G} = (\mathcal{V}, \mathcal{E}, \mathcal{R})$ for multi-relational image modeling. 


\vspace{-0.5mm}
\subsubsection{Model Architecture} \label{sec:app:cv:arch}
\vspace{-0.4mm}

\textbf{General architecture.} In general, we follow the hierarchical image modeling architecture proposed by Swin Transformer~\citep{liu2021swin}, which is verified to be a superior architecture and is applied to many vision backbones~\citep{liu2022convnet,yang2021focal,yang2022focal}. Specifically, the whole model is divided into four stages that (1) reduce the number of patches (\emph{i.e.}, nodes in our graph) to a quarter across consecutive stages, and (2) use increasing number of feature channels $[C, 2C, 4C, 8C]$ for all stages. Each stage contains multiple modeling blocks, where we construct each block with a GRMP layer (Sec.~\ref{sec:method:layer}) for relational message passing and a feed-forward network (FFN)~\citep{vaswani2017attention} for feature transformation. We adjust the number of feature channels and the number of blocks in each stage to get a model series with increasing capacity, \emph{i.e.}, \textbf{{\model}-T}, \textbf{{\model}-S} and \textbf{{\model}-B}. The detailed architectures of these models are displayed in Appendix~\ref{supp:sec:arch}. 

\textbf{Graph construction layers.} To adapt the multi-stage modeling manner, we put a graph construction layer before each modeling stage of {\model}-T/S/B. In this way, based on the locations and representations of the patches fed into each stage, the multi-relational graph $\mathcal{G}$ will be reconstructed to adapt these stage inputs. In particular, along the modeling stages, the edges for medium-range interactions are expected to capture the semantic neighbors on different semantic levels (\emph{e.g.}, from the relevance of low-level features to the relevance of high-level semantics), as studied in Sec.~\ref{sec:exp:visualization}. 


\vspace{-0.3mm}
\subsection{{\model} for Protein Structure Modeling} \label{sec:app:protein}

\vspace{-0.3mm}
\subsubsection{Relational Edges for Short-, Medium- and Long-range Interactions} \label{sec:app:protein:edge}

In this work, we consider the alpha carbon (\emph{i.e.}, C$\alpha$) graph as the representation of protein structure, which is an informative and light-weight summary of the overall protein 3D structure and is widely used in the literature~\citep{gligorijevic2021structure,baldassarre2021graphqa,zhang2022protein} (see Appendix~\ref{supp:sec:protein_intro} for a preliminary introduction to protein structure). In specific, we extract all C$\alpha$s as the node set $\mathcal{V}$ of our graph, which, at this time, is actually a set of separate points in the 3D space, since there is no chemical bond among C$\alpha$s. To describe the multi-range spatial interactions within a protein, we build following relational edges (see Fig.~\ref{fig:protein_edge} for a graphical illustration):
\vspace{-4.0mm}
\begin{itemize}
    \parbox[t]{\dimexpr\textwidth-\leftmargin}{
        \begin{wrapfigure}{R}{0.39\textwidth}
        \vspace{-5.6mm}
        \centering
        \includegraphics[width=0.40\textwidth]{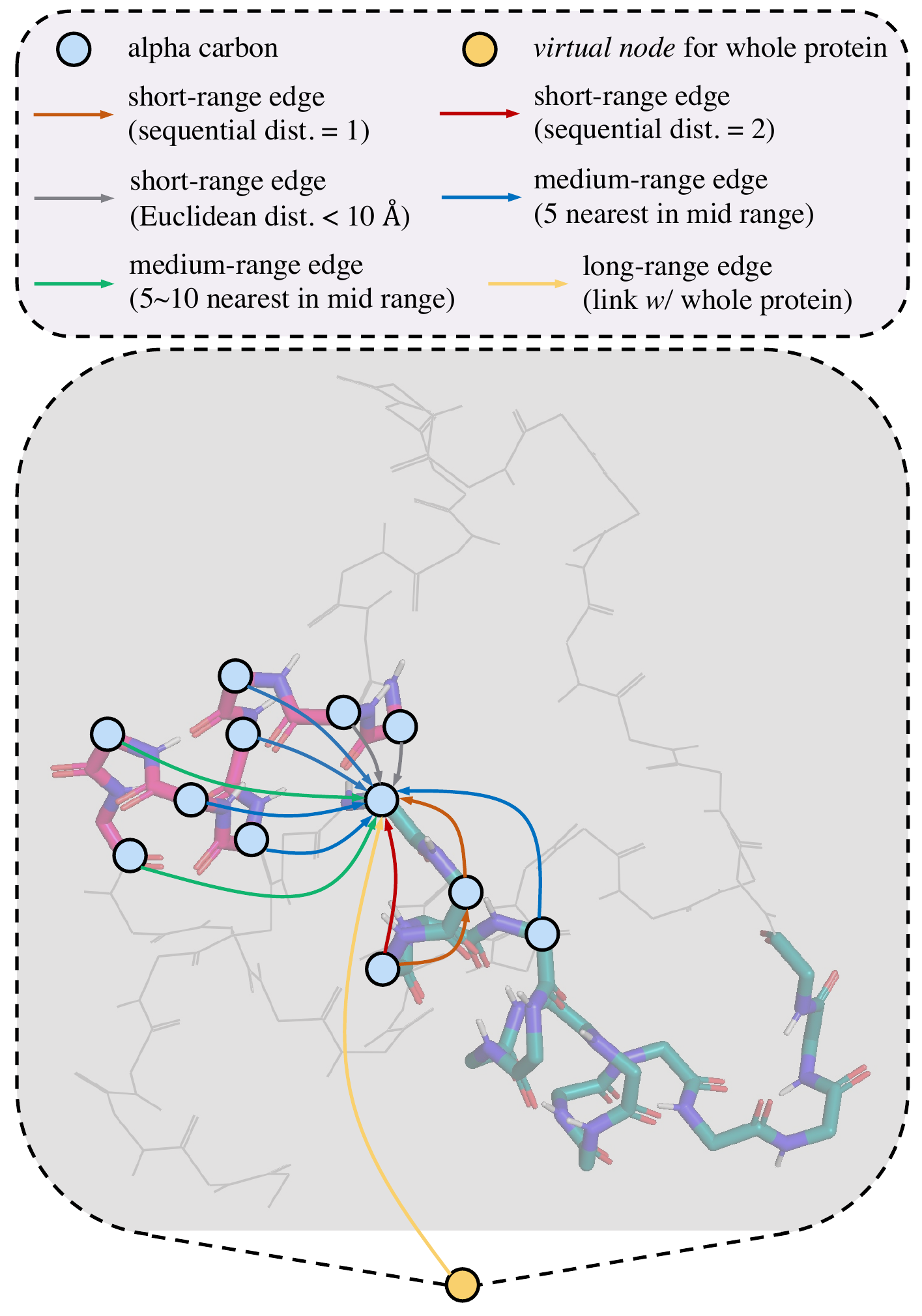}
        \vspace{-9.3mm}
        \caption{Multi-range relational edges for protein C$\alpha$s. \emph{Abbr.}, dist.: distance.} 
        \label{fig:protein_edge}
        \vspace{-3.9mm}
        \end{wrapfigure}
    \item \textbf{Edges for short-range interactions} ($|\mathcal{R}_{\mathrm{short}}| = 6$). We adopt two kinds of short-range edges proposed by~\citet{zhang2022protein}. (1) \emph{Sequential edges} connect the C$\alpha$ nodes that are within the distance of 2 on the protein sequence, where each of the sequential distances \{-2,-1,0,1,2\} is regarded as a single relation (\emph{i.e.}, 5 relations in total). (2) \emph{Radius edges} connect the C$\alpha$ nodes within the Euclidean distance of 10 angstroms, and all radius edges have the same relation. 
    \vspace{-3.4mm}
    \item \textbf{Edges for medium-range interactions} ($|\mathcal{R}_{\mathrm{medium}}| = 2$). To capture medium-range interactions exclusively, for each C$\alpha$ node, we first filter out all its neighbors within the sequential distance of 5 or within the Euclidean distance of 10 angstroms. 
    We then connect it with the remaining nodes that are \emph{5 nearest} and \emph{5$\sim$10 nearest} to it (measured by Euclidean distance), and the connections with these two sets of neighbors are regarded as two different relations.   
    \vspace{0.6mm}
    \item \textbf{Edges for long-range interactions} ($|\mathcal{R}_{\mathrm{long}}| = 1$). To capture the interactions beyond short- and medium-range interactions above, we introduce a \emph{virtual node} representing the whole protein by taking global average pooling over all C$\alpha$ representations, and this virtual node is linked to all C$\alpha$ nodes with a single relation. These edges make each C$\alpha$ aware of the status of all other C$\alpha$s, and thus the long-range interactions beyond short and medium ranges can be captured. 
    }
\end{itemize}
\vspace{-1mm}

We gather all these edges with 9 different relations into the edge set $\mathcal{E}$ and the relation set $\mathcal{R}$, which, together with $\mathcal{V}$, derive the full graph $\mathcal{G} = (\mathcal{V}, \mathcal{E}, \mathcal{R})$ for multi-relational protein structure modeling. 


\vspace{-0.3mm}
\subsubsection{Model Architecture} \label{sec:app:protein:arch}

This work focuses on comparing the graph construction and message passing schemes of {\model} against the SoTA GearNet~\citep{zhang2022protein}, and we thus follow its single-stage model architecture for fair comparison. Specifically, {\model} performs graph construction once before this only modeling stage, and the input node feature is the one-hot encoding of each C$\alpha$'s corresponding amino acid. Upon these inputs, six GRMP layers (Sec.~\ref{sec:method:layer}) are stacked for relational modeling. After each layer, the sum pooling over all C$\alpha$ representations is deemed as the whole-protein representation, and these per-layer protein representations are concatenated to produce the final output. Upon this output, {\model} performs a downstream task by appending a task-specific prediction head. We leave the design of the protein structure encoder with multiple modeling stages as a future work. 

\emph{Note that}, in {\model}, all graph construction and message passing operations rely only on the quantities (\emph{e.g.}, sequential and Euclidean distance) that are invariant to translation, rotation and reflection. Therefore, {\model} satisfies \textbf{E(3)-invariance}~\citep{mumford1994geometric}.


\section{Experiments} \label{sec:exp}
\vspace{-1.5mm}



\subsection{Performance Comparison on Image Modeling} \label{sec:exp:cv}
\vspace{-1mm}

\subsubsection{Baseline Methods} \label{sec:exp:cv:baseline}
\vspace{-1mm}

We do point-by-point comparisons between our {\model} series, the SoTA ConvNeXt~\citep{liu2022convnet} and FocalNet~\citep{yang2022focal} series, and other standard series including Swin Transformer~\citep{liu2021swin}, FocalAtt~\citep{yang2021focal} and ViG~\citep{han2022vision}. For completeness, we also report the results of EffNet~\citep{tan2019efficientnet}, EffNetV2~\citep{tan2021efficientnetv2}, ViT~\citep{dosovitskiy2020image}, DeiT~\citep{touvron2021training}, PVT~\citep{wang2021pyramid}, Mixer~\citep{tolstikhin2021mlp}, gMLP~\citep{liu2021pay} and ResMLP~\citep{touvron2021resmlp} in applicable cases. 


\vspace{-1.8mm}
\subsubsection{Image Classification on ImageNet-1K} \label{sec:exp:cv:cls}
\vspace{-0.7mm}


\begin{wraptable}{r}{0.48\textwidth}
\vspace{-7.7mm}
\caption{ImageNet-1K classification results. We measure throughput on a V100 GPU. {\dag} denotes the model pre-trained on ImageNet-22K. $224^2$ and $384^2$ denote the image size. ``$\uparrow$384'' means fine-tuning on 384\texttimes384 images for 30 epochs.} \label{tab:imagenet-224}
\vspace{0.4mm}
\begin{adjustbox}{max width=1.0\linewidth}
\begin{tabular}{l|ccc|c}
\toprule  
\multirow{2}{*}{\bf{Model}} & \bf{\#Params.} & \bf{FLOPs} & \bf{Throughput} & \bf{Top-1}  \\
& \bf{(M)} & \bf{(G)} & \bf{(imgs/s)} & \bf{Acc (\%)} \\
\midrule \midrule
\multicolumn{5}{c}{\small \bf{Train on ImageNet-1K from scratch}} \\
EffNet-B7 & 66 & 37.0 & 55 & 84.3 \\ 
EffNetV2-L & 120 & 53.0 & 84 & 85.7 \\
\midrule
ViT-S/16 & 22 & 4.6 & 939 & 79.9 \\ 
ViT-B/16 & 87 & 17.6 & 330 & 81.8 \\
DeiT-S/16 & 22 & 4.6 & 979 & 79.8 \\
DeiT-S/16 & 87 & 17.6 & 302 & 81.8 \\
PVT-Small & 25 & 3.8 & 794 & 79.8 \\
PVT-Medium & 44 & 6.7 & 517 & 81.2 \\
PVT-Large & 61 & 9.8 & 352 & 81.7 \\
\midrule
Mixer-B/16 & 60 & 12.7 & 455 & 76.4 \\
gMLP-S & 20 & 4.5 & 785 & 79.6 \\
gMLP-B & 73 & 15.8 & 301 & 81.6 \\
ResMLP-S24 & 30 & 6.0 & 871 & 79.4 \\
ResMLP-B24 & 129 & 23.0 & 61 & 81.0 \\
\midrule
Swin-T & 28 & 4.9 & 760 & 81.2 \\
Pyramid ViG-S & 27 & 4.6 & - & 82.1 \\
FocalAtt-T & 29 & 4.9 & 319 & 82.2 \\
FocalNet$_{\textrm{(LRF)}}$-T & 29 & 4.5 & 696 & \bf{82.3} \\
ConvNeXt-T & 29 & 4.5 & 775 & 82.1 \\
\rowcolor{Gray} 
{\model}-T & 29 & 4.6 & 530 & \bf{82.3} \\
\midrule
Swin-S & 50 & 8.7 & 435 & 83.1 \\
Pyramid ViG-M & 52 & 8.9 & - & 83.1 \\
FocalAtt-S & 51 & 9.4 & 192 & 83.5 \\
FocalNet$_{\textrm{(LRF)}}$-S & 50 & 8.7 & 406 & 83.5 \\
ConvNeXt-S & 50 & 8.7 & 447 & 83.1 \\
\rowcolor{Gray} 
{\model}-S & 50 & 8.8 & 314 & \bf{83.6} \\
\midrule
Swin-B & 88 & 15.4 & 291 & 83.5 \\
Pyramid ViG-B & 93 & 16.8 & - & 83.7 \\
FocalAtt-B & 90 & 16.4 & 138 & 83.8 \\
FocalNet$_{\textrm{(LRF)}}$-B & 89 & 15.4 & 269 & 83.9 \\
ConvNeXt-B & 89 & 15.4 & 292 & 83.8 \\
\rowcolor{Gray} 
{\model}-B & 89 & 15.6 & 224 & \bf{84.1} \\
\midrule
Swin-B$\!\;\uparrow\!\;$384 & 88 & 47.1 & 85 & 84.5 \\
ConvNeXt-B$\!\;\uparrow\!\;$384 & 89 & 45.0 & 96 & 85.1 \\
\rowcolor{Gray} 
{\model}-B$\!\;\uparrow\!\;$384 & 90 & 46.6 & 69 & \bf{85.3} \\
\midrule
\multicolumn{5}{c}{\small \bf{Pre-train on ImageNet-22K \& Fine-tune on ImageNet-1K}} \\
Swin-B{\dag} ($224^2$) & 88 & 15.4 & 291 & 85.2 \\
FocalNet$_{\textrm{(SRF)}}$-B{\dag} ($224^2$) & 88 & 15.3 & 280 & 85.6 \\
ConvNeXt-B{\dag} ($224^2$) & 89 & 15.4 & 292 & \bf{85.8} \\
\rowcolor{Gray}
{\model}-B{\dag} ($224^2$) & 89 & 15.6 & 224 & 85.7 \\
\midrule
Swin-B{\dag} ($384^2$) & 88 & 47.0 & 85 & 86.4 \\
FocalNet$_{\textrm{(SRF)}}$-B{\dag} ($384^2$) & 88 & 44.8 & 94 & 86.5 \\
ConvNeXt-B{\dag} ($384^2$) & 89 & 45.1 & 96 & 86.8 \\
\rowcolor{Gray} 
{\model}-B{\dag} ($384^2$) & 90 & 46.6 & 69 & \bf{87.0} \\
\bottomrule
\end{tabular}
\end{adjustbox}
\vspace{-6.7mm}
\end{wraptable} 


\textbf{Setups.} In this set of experiments, we benchmark the classification performance of different backbones on ImageNet-1K~\citep{deng2009imagenet} (with 1.28M training and 50K validation images from 1,000 classes) in terms of top-1 accuracy. We consider both (1) training on ImageNet-1K from scratch and (2) pretraining on ImageNet-22K followed by ImageNet-1K fine-tuning. For fair comparison, we follow the standard training configurations of Swin Transformer~\citep{liu2021swin} with minor changes. Detailed model and training configurations are stated in Appendix~\ref{supp:sec:setup:cls}.  

\textbf{Results.} In Tab.~\ref{tab:imagenet-224}, for training on ImageNet-1K from scratch, {\model}s outperform or align previous SoTA baselines on $224^2$ image size,~\emph{i.e.}, {\model}-T \emph{v.s.} FocalNet$_{\textrm{(LRF)}}$-T: \textbf{82.3}\% \emph{v.s.} \textbf{82.3}\%; {\model}-S \emph{v.s.} FocalNet$_{\textrm{(LRF)}}$-S: \textbf{83.6}\% \emph{v.s.} 83.5\%; {\model}-B \emph{v.s.} FocalNet$_{\textrm{(LRF)}}$-B: \textbf{84.1}\% \emph{v.s.}~83.9\%. Following Swin Transformer, we lift the resolution to $384^2$ for 30 epochs fine-tuning after training {\model}-B for 300 epochs on the $224^2$ resolution, this model gains \textbf{85.3}\% top-1 accuracy, outperforming ConvNeXt-B. These results demonstrate the effectiveness of {\model}s on modeling images with different resolutions. 

For ImageNet-22K pre-training, {\model}-B is on par with the previous SoTA ConvNeXt-B and clearly outperforms Swin-B. The pre-training on large scale is widely regarded as the strength of the models with few inductive biases like Swin-B; while our results show that the well-designed {\model}-B with more inductive biases could also be effective, aligning with ConvNeXt's finding. 


\textbf{Throughput analysis.} The throughput of {\model} is higher than FocalAtt while lower than FocalNet and ConvNeXt. We point out that 2D convolutions (\emph{i.e.}, the core of FocalNet and ConvNeXt) are well supported by CUDA kernels, while such supports are still ongoing for graph operations~\citep{chen2020fusegnn,min2021large}. {\model}'s further speedup is expected under maturer CUDA supports.


\begin{table}[t]
\begin{spacing}{1.0}
    \centering
    \caption{\small COCO object detection and instance segmentation results with Mask R-CNN~\citep{he2017mask}.}
    \label{tab:coco}
    \begin{adjustbox}{max width=1.0\linewidth}
        \begin{tabular}{l|cc|cccccc|cccccc}
        \toprule
        \multirow{2}{*}{\bf{Model}} & \bf{\#Params.} & \bf{FLOPs} & \multicolumn{6}{c|}{\bf{Mask R-CNN 1{\texttimes}}} & \multicolumn{6}{c}{\bf{Mask R-CNN 3{\texttimes}}} \\ \cline{4-15}
            & \bf{(M)} & \bf{(G)} & $\mathbf{AP^b}$ & $\mathbf{AP^b_{50}}$ & $\mathbf{AP^b_{75}}$ & $\mathbf{AP^m}$ & $\mathbf{AP^m_{50}}$ & $\mathbf{AP^m_{75}}$ & $\mathbf{AP^b}$ & $\mathbf{AP^b_{50}}$ & $\mathbf{AP^b_{75}}$ & $\mathbf{AP^m}$ & $\mathbf{AP^m_{50}}$ & $\mathbf{AP^m_{75}}$ \\
        \midrule \midrule
        PVT-Small & 44.1 & 245 & 40.4 & 62.9 & 43.8 & 37.8 & 60.1 & 40.3 & 43.0 & 65.3 & 46.9 & 39.9 & 62.5 & 42.8 \\
        Swin-T & 47.8 & 264 & 43.7 & 66.6 & 47.7 & 39.8 & 63.3 & 42.7 & 46.0 & 68.1 & 50.3 & 41.6 & 65.1 & 44.9 \\
        FocalAtt-T & 48.8 & 291 & 44.8 & 67.7 & 49.2 & 41.0 & 64.7 & 44.2 & 47.2 & 69.4 & 51.9 & 42.7 & \textbf{66.5} & 45.9 \\
        FocalNet$_{\textrm{(LRF)}}$-T & 48.9 & 268 & \textbf{46.1} & 68.2 & \textbf{50.6} & 41.5 & 65.1 & 44.5 & \textbf{48.0} & \textbf{69.7} & \textbf{53.0} & \textbf{42.9} & \textbf{66.5} & \textbf{46.1} \\
        \rowcolor{Gray}
        {\model}-T & 49.8 & 271 & \textbf{46.1} & \textbf{68.7} & 50.5 & \textbf{41.6} & \textbf{65.5} & \textbf{44.6} & 47.8 & 69.5 & 52.3 & \textbf{42.9} & \textbf{66.5} & \textbf{46.1} \\
        \midrule
        PVT-Medium & 63.9 & 302 & 42.0 & 64.4 & 45.6 & 39.0 & 61.6 & 42.1 & 44.2 & 66.0 & 48.2 & 40.5 & 63.1 & 43.5 \\
        Swin-S & 69.1 & 354 & 46.5 & 68.7 & 51.3 & 42.1 & 65.8 & 45.2 & 48.5 & 70.2 & 53.5 & 43.3 & 67.3 & 46.6 \\
        FocalAtt-S & 71.2 & 401 & 47.4 & 69.8 & 51.9 & 42.8 & 66.6 & 46.1 & 48.8 & 70.5 & 53.6 & 43.8 & 67.7 & 47.2 \\
        FocalNet$_{\textrm{(LRF)}}$-S & 72.3 & 365 & 48.3 & \textbf{70.5} & 53.1 & 43.1 & \textbf{67.4} & \textbf{46.2} & 49.3 & \textbf{70.7} & 54.2 & 43.8 & \textbf{67.9} & 47.4 \\
        \rowcolor{Gray}
        {\model}-S & 72.8 & 364 & \textbf{48.4} & \textbf{70.5} & \textbf{53.2} & \textbf{43.2} & \textbf{67.4} & \textbf{46.2} & \textbf{49.4} & \textbf{70.7} & \textbf{54.5} & \textbf{44.0} & 67.6 & \textbf{47.5} \\
        \midrule
        PVT-Large & 81.0 & 364 & 42.9 & 65.0 & 46.6 & 39.5 & 61.9 & 42.5 & 44.5 & 66.0 & 48.3 & 40.7 & 63.4 & 43.7 \\
        Swin-B & 107.1 & 497 & 46.9 & 69.2 & 51.6 & 42.3 & 66.0 & 45.5 & 48.5 & 69.8 & 53.2 & 43.4 & 66.8 & 46.9 \\
        FocalAtt-B & 110.0 & 533 & 47.8 & 70.2 & 52.5 & 43.2 & 67.3 & 46.5 & 49.0 & 70.1 & 53.6 & 43.7 & 67.6 & 47.0 \\ 
        FocalNet$_{\textrm{(LRF)}}$-B & 111.4 & 507 & 49.0 & 70.9 & 53.9 & 43.5 & 67.9 & 46.7 & 49.8 & 70.9 & 54.6 & 44.1 & 68.2 & 47.2 \\
        \rowcolor{Gray}
        {\model}-B & 112.1 & 506 & \bf{49.3} & \bf{71.8} & \bf{54.0} & \bf{43.9} & \bf{68.2} & \bf{47.2} & \bf{50.1} & \bf{71.5} & \bf{55.1} & \bf{44.5} & \bf{68.7} & \bf{47.8} \\
        \bottomrule
        \end{tabular}
    \end{adjustbox}
\end{spacing}
\vspace{-3.0mm}
\end{table}


\vspace{-1.8mm}
\subsubsection{Object Detection on COCO} \label{sec:exp:cv:det}
\vspace{-0.7mm}

\textbf{Setups.} This experiment benchmarks the object detection and instance segmentation performance on COCO 2017~\citep{lin2014microsoft}. All models are trained on 118K training images and evaluated on 5K validation images. Two standard training schedules, \emph{i.e.}, the 1{\texttimes} schedule with 12 epochs and the 3{\texttimes} schedule with 36 epochs, are used for benchmarking. Detailed setups are stated in Appendix~\ref{supp:sec:setup:det}.

\textbf{Results.} In Tab.~\ref{tab:coco}, {\model} performs comparably to FocalNet$_{\textrm{(LRF)}}$ on the tiny and small model scales. We can observe the superiority of {\model}-B over FocalNet$_{\textrm{(LRF)}}$-B on the base model scale (better performance on all 12 metrics). The base-scale {\model}-B owns $[2,2,18,2]$ modeling blocks (more than {\model}-T) and $[128,256,512,1024]$ feature channels (more than {\model}-S) for four modeling stages. Therefore, \emph{larger message passing hops} (achieved by more modeling blocks) coupled with larger model width favor {\model}'s performance on high-resolution dense prediction tasks. 



\begin{wraptable}{r}{0.42\textwidth}
\vspace{-15.2mm}
\caption{ADE20K semantic segmentation results with UperNet~\citep{xiao2018unified}.} 
\label{tab:ade20k}
\begin{adjustbox}{max width=1.0\linewidth}
\begin{tabular}{l|cc|cc}
\toprule  
\bf{Model} & \bf{\#Params. (M)} & \bf{FLOPs (G)} & \bf{mIoU} & \bf{+MS}  \\
\midrule \midrule
Swin-T & 60 & 941 & 44.5 & 45.8 \\
FocalAtt-T & 62 & 998 & 45.8 & 47.0 \\
ConvNeXt-T & 60 & 939 & - & 46.7 \\
FocalNet$_{\textrm{(LRF)}}$-T & 61 & 949 & 46.8 & 47.8 \\
\rowcolor{Gray}
{\model}-T & 62 & 948 & \bf{47.2} & \bf{48.4} \\
\midrule
Swin-S & 81 & 1038 & 47.6 & 49.5 \\
FocalAtt-S & 85 & 1130 & 48.0 & 50.0 \\
ConvNeXt-S & 82 & 1027 & - & 49.6 \\
FocalNet$_{\textrm{(LRF)}}$-S & 84 & 1044 & 49.1 & 50.1 \\
\rowcolor{Gray}
{\model}-S & 85 & 1042 & \bf{49.8} & \bf{50.8} \\
\midrule
Swin-B & 121 & 1188 & 48.1 & 49.7 \\
FocalAtt-B & 126 & 1354 & 49.0 & 50.5 \\
ConvNeXt-B & 122 & 1170 & - & 49.9 \\ 
FocalNet$_{\textrm{(LRF)}}$-B & 126 & 1192 & 50.5 & 51.4 \\
\rowcolor{Gray}
{\model}-B & 126 & 1190 & \bf{50.7} & \bf{51.8} \\
\bottomrule
\end{tabular}
\end{adjustbox}
\vspace{-6.7mm}
\end{wraptable} 


\vspace{-1.7mm}
\subsubsection{Semantic Segmentation on ADE20K} \label{sec:exp:cv:seg}
\vspace{-0.6mm}

\textbf{Setups.} In this experiment, we benchmark the semantic segmentation performance of different backbones on ADE20K~\citep{zhou2017scene} which contains 20K training, 2K validation and 3K test images. The mIoU metrics under both single- and multi-scale (MS) evaluation are reported. We provide more details in Appendix~\ref{supp:sec:setup:seg}.

\textbf{Results.} Tab.~\ref{tab:ade20k} reports all results. It can be observed that {\model}-T, {\model}-S and {\model}-B achieve the best performance on their corresponding model scales under both evaluation metrics. Such consistent performance gains verify the effectiveness of {\model} on the dense prediction tasks that require to model fine-grained semantics and long-range interactions. 


\vspace{-1.7mm}
\subsection{Performance Comparison on Protein Structure Modeling} \label{sec:exp:protein}
\vspace{-0.6mm}


\begin{wraptable}{r}{0.42\textwidth}
\vspace{-14.7mm}
\begin{spacing}{1.03}
\caption{$\mathrm{F}_{\mathrm{max}}$ results on EC and GO protein function prediction benchmarks.} 
\label{tab:ec_go}
\begin{adjustbox}{max width=1.0\linewidth}
\begin{tabular}{l|cccc}
\toprule  
\bf{Model} & \bf{EC} & \bf{GO-BP} & \bf{GO-MF} & \bf{GO-CC}  \\
\midrule \midrule
3DCNN\_MQA & 0.077 & 0.240 & 0.147 & 0.305 \\
GCN & 0.320 & 0.252 & 0.195 & 0.329 \\
GAT & 0.368 & 0.284 & 0.317 & 0.385 \\
GVP & 0.489 & 0.326 & 0.426 & 0.420 \\
GraphQA & 0.509 & 0.308 & 0.329 & 0.413 \\
New IEConv & 0.735 & 0.374 & 0.544 & \bf{0.444} \\
\midrule
GearNet & 0.730 & 0.356 & 0.503 & 0.414 \\
\rowcolor{Gray}
{\model} & \bf{0.768} & \bf{0.437} & \bf{0.563} & \bf{0.421} \\ 
\midrule
GearNet-Edge & 0.810 & 0.403 & 0.580 & 0.450 \\
\rowcolor{Gray}
{\model}-Edge & \bf{0.829} & \bf{0.456} & \bf{0.592} & \bf{0.453} \\ 
\bottomrule
\end{tabular}
\end{adjustbox}
\end{spacing}
\vspace{-6.4mm}
\end{wraptable} 


\subsubsection{Baseline Methods} \label{sec:exp:protein:baseline}
\vspace{-1mm}

We compare with the SoTA GearNet \citep{zhang2017mixup} under two settings, \emph{i.e.}, with and without edge message passing (``-Edge'' in Tab.~\ref{tab:ec_go}). We also include other baselines, \emph{i.e.}, 3DCNN\_MQA~\citep{derevyanko2018deep}, GCN~\citep{kipf2016semi}, GAT~\citep{velivckovic2017graph}, GVP~\citep{jing2021equivariant}, GraphQA~\citep{baldassarre2021graphqa} and New IEConv~\citep{hermosilla2022contrastive}, for complete comparisons.


\begin{figure}[t]
\centering
    \includegraphics[width=0.97\linewidth]{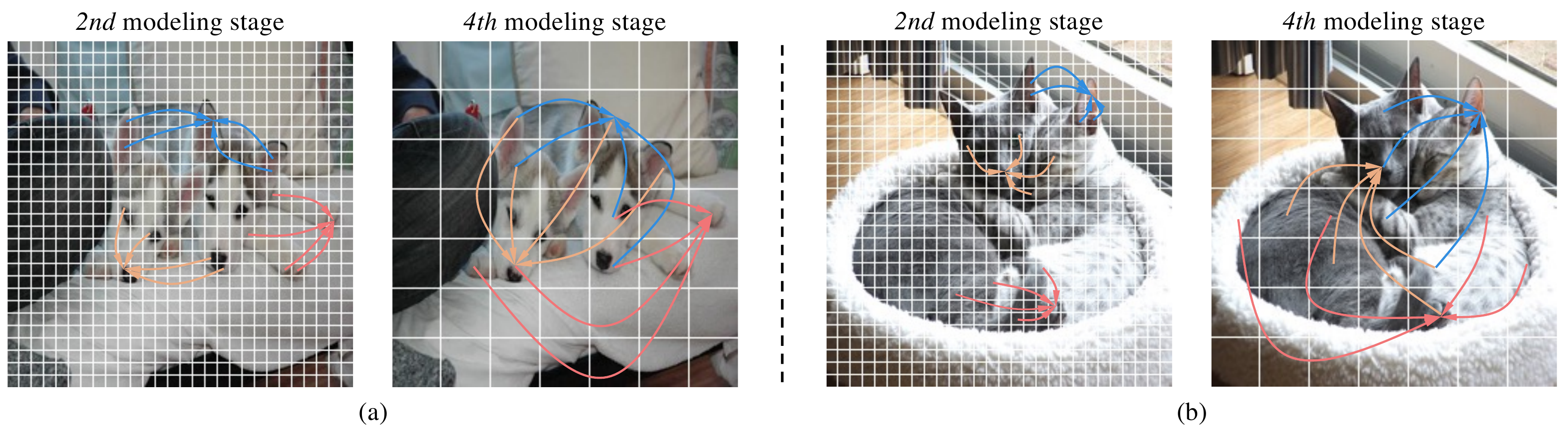}
    \vspace{-4.0mm}
    \caption{\small Medium-range edges built by {\model}-T (we use different colors for different selected target nodes).   
    }
    \label{fig:visualization}
    \vspace{-2.8mm}
\end{figure}


\vspace{-1.7mm}
\subsubsection{Protein Function Prediction} \label{sec:exp:protein:func}
\vspace{-0.6mm}

\textbf{Setups.} This set of experiments compare different protein structure encoders on the EC~\citep{gligorijevic2021structure} and GO~\citep{gligorijevic2021structure} protein function prediction benchmarks. We follow GearNet to report the protein-centric maximum F-score $\mathrm{F}_{\mathrm{max}}$, a commonly-used metric in CAFA challenges~\citep{radivojac2013large}. More dataset, model and training details are in Appendix~\ref{supp:sec:setup:protein}.

\textbf{Results.} In Tab.~\ref{tab:ec_go}, we can observe that {\model} consistently outperforms GearNet on all four tasks, and the performance gains preserve after involving edge message passing (details of edge message passing are stated in Appendix~\ref{supp:sec:setup:protein}). Since {\model} follows the single-stage model architecture of GearNet, we can conclude the effectiveness of \emph{medium- and long-range interaction modeling} and \emph{GRMP-based multi-relational modeling}, which are novel modeling mechanisms in {\model}.  


\begin{wraptable}{r}{0.36\textwidth}
\begin{spacing}{0.972}
\vspace{-14.6mm}
\small
\caption{\small Ablation study of multi-range edges on ImageNet-1K with {\model}-T.}
\label{tab:ablation:edge}
\vspace{-1.8mm}
\center
\begin{adjustbox}{max width=1.0\linewidth}
\begin{tabular}{ccc|c}
\toprule
\bf{short} & \bf{medium} & \bf{long} & \bf{Top-1 Acc (\%)} \\ 
\midrule
\checkmark & & & 80.7 \\
& \checkmark & & 79.3 \\
& & \checkmark & \bf{81.7} \\
\midrule
\checkmark & \checkmark &  & 81.5 \\
\checkmark & & \checkmark & \bf{82.0} \\
& \checkmark & \checkmark & \bf{82.0} \\
\midrule
\checkmark & \checkmark & \checkmark & \bf{82.3} \\
\bottomrule
\end{tabular}
\end{adjustbox}
\end{spacing}
\vspace{-2mm}
\end{wraptable} 


\vspace{-1.85mm}
\subsection{Ablation Study} \label{sec:exp:ablation}
\vspace{-1.03mm}

\textbf{Effect of multi-range relational edges.} In Tab.~\ref{tab:ablation:edge}, we evaluate {\model}-T on ImageNet-1K with different ranges of edges. When using a single range, the model with long-range edges achieves the highest accuracy 81.7\%, which verifies the importance of capturing long-range interactions in image classification. By further adding short- or medium-range edges, the performance is promoted to 82.0\%, where more fine-grained local interactions are captured. By using all three ranges of edges, the full model of {\model}-T obtains the 82.3\% accuracy, which proves the complementarity of short-, medium- and long-range edges. Ablation study for protein structure is in Appendix~\ref{supp:sec:ablation:edge}.


\begin{wraptable}{r}{0.55\textwidth}
\vspace{-8.2mm}
\begin{spacing}{1.04}
\caption{\small Ablation study of multi-relational modeling layer on ImageNet-1K with {\model}-T.}
\label{tab:ablation:layer}
\begin{adjustbox}{max width=1.0\linewidth}
\begin{tabular}{lc|ccc|c}
\toprule
\multirow{2}{*}{\bf{Layer}} & \bf{Hidden} & \bf{\#Params.} & \bf{FLOPs} & \bf{Throughput} & \bf{Top-1}  \\
& \bf{Dimensions} & \bf{(M)} & \bf{(G)} & \bf{(imgs/s)} & \bf{Acc (\%)} \\
\midrule
RGConv & $[84,168,336,672]$ & 28.8 & 4.6 & 541.8 & 81.5 \\
\rowcolor{Gray}
GRMP & $[96,192,384,768]$ & 28.8 & 4.6 & 530.3 & \bf{82.3} \\
RGConv & $[96,192,384,768]$ & 37.3 & 5.9 & 451.2 & 82.2 \\
\rowcolor{Gray}
GRMP & $[108,216,432,864]$ & 36.3 & 5.8 & 444.5 & \bf{82.7} \\
\bottomrule
\end{tabular}
\end{adjustbox}
\end{spacing}
\vspace{-5mm}
\end{wraptable} 


\textbf{Effect of GRMP layer.} In Tab.~\ref{tab:ablation:layer}, we compare between RGConv and GRMP under the comparable parameter number, FLOPs and throughput. (1) GRMP's dimensions are first set as $[96,192,384,768]$ in four stages. To reach comparable cost, RGConv can only have the dimensions of $[84,168,336,672]$ and achieves a lower accuracy 81.5\% than GRMP's 82.3\%. (2) After increasing RGConv's dimensions to $[96,192,384,768]$, it aligns GRMP's performance while introduces more cost (1.3G more FLOPs). Under comparable cost, GRMP can have $[108,216,432,864]$ dimensions, leading to a higher accuracy 82.7\%. These results demonstrate the better efficiency-performance trade-off gained by GRMP. Ablation study for protein structure modeling is in Appendix~\ref{supp:sec:ablation:layer}.


\vspace{-1.85mm}
\subsection{Visualization} \label{sec:exp:visualization}
\vspace{-1.08mm}

Fig.~\ref{fig:visualization} displays some medium-range edges built by the {\model}-T trained on ImageNet-1K. The edges for the 2nd stage connect the patches with similar low-level features (\emph{e.g.}, the patches of red dog ears in Fig.~\ref{fig:visualization}(a)), while the edges for the 4th stage connect semantically relevant patches (\emph{e.g.}, different body parts of two dogs in Fig.~\ref{fig:visualization}(a)), which shows {\model}-T's hierarchical image modeling ability.  


\vspace{-1.6mm}
\section{Conclusions and Future Work} \label{sec:conclusion}
\vspace{-1.3mm}

This work proposes the {\model} to model spatial multi-relational data like image patches and protein alpha carbons. It builds relational edges on multiple spatial ranges to describe the interactions in the data. It uses the gated relational message passing layer to model the built multi-relational graph, which can efficiently adapt to large data and model scales. The instantiations of {\model} have gained superior performance on various image and protein structure modeling tasks. 

\vspace{-0.3mm}
In future works, we will adapt {\model} to more tasks of other domains like 3D point cloud modeling for object and scene understanding, and we will explore a general hierarchical multi-relational modeling method for the data from various domains. 



\newpage
\bibliography{reference}

\begin{thebibliography}{63}
\providecommand{\natexlab}[1]{#1}
\providecommand{\url}[1]{\texttt{#1}}
\expandafter\ifx\csname urlstyle\endcsname\relax
  \providecommand{\doi}[1]{doi: #1}\else
  \providecommand{\doi}{doi: \begingroup \urlstyle{rm}\Url}\fi

\bibitem[Baek et~al.(2021)Baek, DiMaio, Anishchenko, Dauparas, Ovchinnikov,
  Lee, Wang, Cong, Kinch, Schaeffer, et~al.]{baek2021accurate}
Minkyung Baek, Frank DiMaio, Ivan Anishchenko, Justas Dauparas, Sergey
  Ovchinnikov, Gyu~Rie Lee, Jue Wang, Qian Cong, Lisa~N Kinch, R~Dustin
  Schaeffer, et~al.
\newblock Accurate prediction of protein structures and interactions using a
  three-track neural network.
\newblock \emph{Science}, 373\penalty0 (6557):\penalty0 871--876, 2021.

\bibitem[Balcilar et~al.(2020)Balcilar, Renton, H{\'e}roux, Ga{\"u}z{\`e}re,
  Adam, and Honeine]{balcilar2020spectral}
Muhammet Balcilar, Guillaume Renton, Pierre H{\'e}roux, Benoit Ga{\"u}z{\`e}re,
  S{\'e}bastien Adam, and Paul Honeine.
\newblock Spectral-designed depthwise separable graph neural networks.
\newblock In \emph{ICML Workshop on Graph Representation Learning and Beyond
  (GRL+ 2020)}, 2020.

\bibitem[Baldassarre et~al.(2021)Baldassarre, Men{\'e}ndez~Hurtado, Elofsson,
  and Azizpour]{baldassarre2021graphqa}
Federico Baldassarre, David Men{\'e}ndez~Hurtado, Arne Elofsson, and Hossein
  Azizpour.
\newblock Graphqa: protein model quality assessment using graph convolutional
  networks.
\newblock \emph{Bioinformatics}, 37\penalty0 (3):\penalty0 360--366, 2021.

\bibitem[Berman et~al.(2000)Berman, Westbrook, Feng, Gilliland, Bhat, Weissig,
  Shindyalov, and Bourne]{berman2000protein}
Helen~M Berman, John Westbrook, Zukang Feng, Gary Gilliland, Talapady~N Bhat,
  Helge Weissig, Ilya~N Shindyalov, and Philip~E Bourne.
\newblock The protein data bank.
\newblock \emph{Nucleic Acids Research}, 28\penalty0 (1):\penalty0 235--242,
  2000.

\bibitem[Bordes et~al.(2013)Bordes, Usunier, Garcia-Duran, Weston, and
  Yakhnenko]{bordes2013translating}
Antoine Bordes, Nicolas Usunier, Alberto Garcia-Duran, Jason Weston, and Oksana
  Yakhnenko.
\newblock Translating embeddings for modeling multi-relational data.
\newblock \emph{Advances in Neural Information Processing Systems}, 2013.

\bibitem[Chen et~al.(2019{\natexlab{a}})Chen, Wang, Pang, Cao, Xiong, Li, Sun,
  Feng, Liu, Xu, et~al.]{mmdetection}
Kai Chen, Jiaqi Wang, Jiangmiao Pang, Yuhang Cao, Yu~Xiong, Xiaoxiao Li,
  Shuyang Sun, Wansen Feng, Ziwei Liu, Jiarui Xu, et~al.
\newblock Mmdetection: Open mmlab detection toolbox and benchmark.
\newblock \emph{arXiv preprint arXiv:1906.07155}, 2019{\natexlab{a}}.

\bibitem[Chen et~al.(2019{\natexlab{b}})Chen, Rohrbach, Yan, Shuicheng, Feng,
  and Kalantidis]{chen2019graph}
Yunpeng Chen, Marcus Rohrbach, Zhicheng Yan, Yan Shuicheng, Jiashi Feng, and
  Yannis Kalantidis.
\newblock Graph-based global reasoning networks.
\newblock In \emph{IEEE/CVF Conference on Computer Vision and Pattern
  Recognition}, 2019{\natexlab{b}}.

\bibitem[Chen et~al.(2020)Chen, Yan, Zhu, Deng, Li, Li, and
  Xie]{chen2020fusegnn}
Zhaodong Chen, Mingyu Yan, Maohua Zhu, Lei Deng, Guoqi Li, Shuangchen Li, and
  Yuan Xie.
\newblock fusegnn: accelerating graph convolutional neural network training on
  gpgpu.
\newblock In \emph{IEEE/ACM International Conference On Computer Aided Design},
  2020.

\bibitem[Contributors(2020)]{mmsegmentation}
MMSegmentation Contributors.
\newblock Mmsegmentation: Openmmlab semantic segmentation toolbox and
  benchmark.
\newblock \emph{Availabe online: https://github. com/open-mmlab/mmsegmentation
  (accessed on 18 May 2022)}, 2020.

\bibitem[Deng et~al.(2009)Deng, Dong, Socher, Li, Li, and
  Fei-Fei]{deng2009imagenet}
Jia Deng, Wei Dong, Richard Socher, Li-Jia Li, Kai Li, and Li~Fei-Fei.
\newblock Imagenet: A large-scale hierarchical image database.
\newblock In \emph{IEEE/CVF Conference on Computer Vision and Pattern
  Recognition}, 2009.

\bibitem[Derevyanko et~al.(2018)Derevyanko, Grudinin, Bengio, and
  Lamoureux]{derevyanko2018deep}
Georgy Derevyanko, Sergei Grudinin, Yoshua Bengio, and Guillaume Lamoureux.
\newblock Deep convolutional networks for quality assessment of protein folds.
\newblock \emph{Bioinformatics}, 34\penalty0 (23):\penalty0 4046--4053, 2018.

\bibitem[Dettmers et~al.(2018)Dettmers, Minervini, Stenetorp, and
  Riedel]{dettmers2018convolutional}
Tim Dettmers, Pasquale Minervini, Pontus Stenetorp, and Sebastian Riedel.
\newblock Convolutional 2d knowledge graph embeddings.
\newblock In \emph{AAAI Conference on Artificial Intelligence}, 2018.

\bibitem[Ding et~al.(2022)Ding, Zhang, Han, and Ding]{ding2022scaling}
Xiaohan Ding, Xiangyu Zhang, Jungong Han, and Guiguang Ding.
\newblock Scaling up your kernels to 31x31: Revisiting large kernel design in
  cnns.
\newblock In \emph{IEEE/CVF Conference on Computer Vision and Pattern
  Recognition}, 2022.

\bibitem[Dosovitskiy et~al.(2020)Dosovitskiy, Beyer, Kolesnikov, Weissenborn,
  Zhai, Unterthiner, Dehghani, Minderer, Heigold, Gelly,
  et~al.]{dosovitskiy2020image}
Alexey Dosovitskiy, Lucas Beyer, Alexander Kolesnikov, Dirk Weissenborn,
  Xiaohua Zhai, Thomas Unterthiner, Mostafa Dehghani, Matthias Minderer, Georg
  Heigold, Sylvain Gelly, et~al.
\newblock An image is worth 16x16 words: Transformers for image recognition at
  scale.
\newblock \emph{arXiv preprint arXiv:2010.11929}, 2020.

\bibitem[Gainza et~al.(2020)Gainza, Sverrisson, Monti, Rodola, Boscaini,
  Bronstein, and Correia]{gainza2020deciphering}
Pablo Gainza, Freyr Sverrisson, Frederico Monti, Emanuele Rodola, D~Boscaini,
  MM~Bronstein, and BE~Correia.
\newblock Deciphering interaction fingerprints from protein molecular surfaces
  using geometric deep learning.
\newblock \emph{Nature Methods}, 17\penalty0 (2):\penalty0 184--192, 2020.

\bibitem[Gligorijevi{\'c} et~al.(2021)Gligorijevi{\'c}, Renfrew, Kosciolek,
  Leman, Berenberg, Vatanen, Chandler, Taylor, Fisk, Vlamakis,
  et~al.]{gligorijevic2021structure}
Vladimir Gligorijevi{\'c}, P~Douglas Renfrew, Tomasz Kosciolek, Julia~Koehler
  Leman, Daniel Berenberg, Tommi Vatanen, Chris Chandler, Bryn~C Taylor, Ian~M
  Fisk, Hera Vlamakis, et~al.
\newblock Structure-based protein function prediction using graph convolutional
  networks.
\newblock \emph{Nature Communications}, 12\penalty0 (1):\penalty0 1--14, 2021.

\bibitem[Han et~al.(2022)Han, Wang, Guo, Tang, and Wu]{han2022vision}
Kai Han, Yunhe Wang, Jianyuan Guo, Yehui Tang, and Enhua Wu.
\newblock Vision gnn: An image is worth graph of nodes.
\newblock \emph{arXiv preprint arXiv:2206.00272}, 2022.

\bibitem[Harary \& Norman(1960)Harary and Norman]{harary1960some}
Frank Harary and Robert~Z Norman.
\newblock Some properties of line digraphs.
\newblock \emph{Rendiconti del circolo matematico di palermo}, 9\penalty0
  (2):\penalty0 161--168, 1960.

\bibitem[Harms \& Thornton(2010)Harms and Thornton]{harms2010analyzing}
Michael~J Harms and Joseph~W Thornton.
\newblock Analyzing protein structure and function using ancestral gene
  reconstruction.
\newblock \emph{Current Opinion in Structural Biology}, 20\penalty0
  (3):\penalty0 360--366, 2010.

\bibitem[He et~al.(2016)He, Zhang, Ren, and Sun]{he2016deep}
Kaiming He, Xiangyu Zhang, Shaoqing Ren, and Jian Sun.
\newblock Deep residual learning for image recognition.
\newblock In \emph{IEEE/CVF Conference on Computer Vision and Pattern
  Recognition}, 2016.

\bibitem[He et~al.(2017)He, Gkioxari, Doll{\'a}r, and Girshick]{he2017mask}
Kaiming He, Georgia Gkioxari, Piotr Doll{\'a}r, and Ross Girshick.
\newblock Mask r-cnn.
\newblock In \emph{IEEE/CVF International Conference on Computer Vision}, 2017.

\bibitem[Hermosilla \& Ropinski(2022)Hermosilla and
  Ropinski]{hermosilla2022contrastive}
Pedro Hermosilla and Timo Ropinski.
\newblock Contrastive representation learning for 3d protein structures.
\newblock \emph{arXiv preprint arXiv:2205.15675}, 2022.

\bibitem[Jing et~al.(2021)Jing, Eismann, Soni, and Dror]{jing2021equivariant}
Bowen Jing, Stephan Eismann, Pratham~N. Soni, and Ron~O. Dror.
\newblock Learning from protein structure with geometric vector perceptrons.
\newblock In \emph{International Conference on Learning Representations}, 2021.
\newblock URL \url{https://openreview.net/forum?id=1YLJDvSx6J4}.

\bibitem[Jumper et~al.(2021)Jumper, Evans, Pritzel, Green, Figurnov,
  Ronneberger, Tunyasuvunakool, Bates, {\v{Z}}{\'\i}dek, Potapenko,
  et~al.]{jumper2021highly}
John Jumper, Richard Evans, Alexander Pritzel, Tim Green, Michael Figurnov,
  Olaf Ronneberger, Kathryn Tunyasuvunakool, Russ Bates, Augustin
  {\v{Z}}{\'\i}dek, Anna Potapenko, et~al.
\newblock Highly accurate protein structure prediction with alphafold.
\newblock \emph{Nature}, 596\penalty0 (7873):\penalty0 583--589, 2021.

\bibitem[Kipf \& Welling(2016)Kipf and Welling]{kipf2016semi}
Thomas~N Kipf and Max Welling.
\newblock Semi-supervised classification with graph convolutional networks.
\newblock \emph{arXiv preprint arXiv:1609.02907}, 2016.

\bibitem[Li et~al.(2014)Li, Zhang, Meng, and Li]{li2014recommendation}
Jing Li, Lingling Zhang, Fan Meng, and Fenhua Li.
\newblock Recommendation algorithm based on link prediction and domain
  knowledge in retail transactions.
\newblock \emph{Procedia Computer Science}, 31:\penalty0 875--881, 2014.

\bibitem[Li et~al.(2021)Li, Zhang, Liu, Dai, and Wu]{li2021dimensionwise}
Qimai Li, Xiaotong Zhang, Han Liu, Quanyu Dai, and Xiao-Ming Wu.
\newblock Dimensionwise separable 2-d graph convolution for unsupervised and
  semi-supervised learning on graphs.
\newblock In \emph{ACM SIGKDD Conference on Knowledge Discovery \& Data
  Mining}, 2021.

\bibitem[Lin et~al.(2014)Lin, Maire, Belongie, Hays, Perona, Ramanan,
  Doll{\'a}r, and Zitnick]{lin2014microsoft}
Tsung-Yi Lin, Michael Maire, Serge Belongie, James Hays, Pietro Perona, Deva
  Ramanan, Piotr Doll{\'a}r, and C~Lawrence Zitnick.
\newblock Microsoft coco: Common objects in context.
\newblock In \emph{European Conference on Computer Vision}, 2014.

\bibitem[Lin et~al.(2017)Lin, Doll{\'a}r, Girshick, He, Hariharan, and
  Belongie]{fpn}
Tsung-Yi Lin, Piotr Doll{\'a}r, Ross Girshick, Kaiming He, Bharath Hariharan,
  and Serge Belongie.
\newblock Feature pyramid networks for object detection.
\newblock In \emph{IEEE/CVF Conference on Computer Vision and Pattern
  Recognition}, 2017.

\bibitem[Liu et~al.(2021{\natexlab{a}})Liu, Dai, So, and Le]{liu2021pay}
Hanxiao Liu, Zihang Dai, David So, and Quoc~V Le.
\newblock Pay attention to mlps.
\newblock \emph{Advances in Neural Information Processing Systems},
  2021{\natexlab{a}}.

\bibitem[Liu et~al.(2021{\natexlab{b}})Liu, Lin, Cao, Hu, Wei, Zhang, Lin, and
  Guo]{liu2021swin}
Ze~Liu, Yutong Lin, Yue Cao, Han Hu, Yixuan Wei, Zheng Zhang, Stephen Lin, and
  Baining Guo.
\newblock Swin transformer: Hierarchical vision transformer using shifted
  windows.
\newblock In \emph{IEEE/CVF International Conference on Computer Vision},
  2021{\natexlab{b}}.

\bibitem[Liu et~al.(2022)Liu, Mao, Wu, Feichtenhofer, Darrell, and
  Xie]{liu2022convnet}
Zhuang Liu, Hanzi Mao, Chao-Yuan Wu, Christoph Feichtenhofer, Trevor Darrell,
  and Saining Xie.
\newblock A convnet for the 2020s.
\newblock In \emph{IEEE/CVF Conference on Computer Vision and Pattern
  Recognition}, 2022.

\bibitem[Loshchilov \& Hutter(2017)Loshchilov and
  Hutter]{loshchilov2017decoupled}
Ilya Loshchilov and Frank Hutter.
\newblock Decoupled weight decay regularization.
\newblock \emph{arXiv preprint arXiv:1711.05101}, 2017.

\bibitem[Min et~al.(2021)Min, Wu, Huang, Hidayeto{\u{g}}lu, Xiong, Ebrahimi,
  Chen, and Hwu]{min2021large}
Seung~Won Min, Kun Wu, Sitao Huang, Mert Hidayeto{\u{g}}lu, Jinjun Xiong, Eiman
  Ebrahimi, Deming Chen, and Wen-mei Hwu.
\newblock Large graph convolutional network training with gpu-oriented data
  communication architecture.
\newblock \emph{arXiv preprint arXiv:2103.03330}, 2021.

\bibitem[Mumford et~al.(1994)Mumford, Fogarty, and
  Kirwan]{mumford1994geometric}
David Mumford, John Fogarty, and Frances Kirwan.
\newblock \emph{Geometric invariant theory}, volume~34.
\newblock Springer Science \& Business Media, 1994.

\bibitem[Paszke et~al.(2017)Paszke, Gross, Chintala, Chanan, Yang, DeVito, Lin,
  Desmaison, Antiga, and Lerer]{pytorch}
Adam Paszke, Sam Gross, Soumith Chintala, Gregory Chanan, Edward Yang, Zachary
  DeVito, Zeming Lin, Alban Desmaison, Luca Antiga, and Adam Lerer.
\newblock Automatic differentiation in pytorch.
\newblock In \emph{NeurIPS Workshop}, 2017.

\bibitem[Radivojac et~al.(2013)Radivojac, Clark, Oron, Schnoes, Wittkop,
  Sokolov, Graim, Funk, Verspoor, Ben-Hur, et~al.]{radivojac2013large}
Predrag Radivojac, Wyatt~T Clark, Tal~Ronnen Oron, Alexandra~M Schnoes, Tobias
  Wittkop, Artem Sokolov, Kiley Graim, Christopher Funk, Karin Verspoor, Asa
  Ben-Hur, et~al.
\newblock A large-scale evaluation of computational protein function
  prediction.
\newblock \emph{Nature methods}, 10\penalty0 (3):\penalty0 221--227, 2013.

\bibitem[Schlichtkrull et~al.(2018)Schlichtkrull, Kipf, Bloem, Berg, Titov, and
  Welling]{schlichtkrull2018modeling}
Michael Schlichtkrull, Thomas~N Kipf, Peter Bloem, Rianne van~den Berg, Ivan
  Titov, and Max Welling.
\newblock Modeling relational data with graph convolutional networks.
\newblock In \emph{European Semantic Web Conference}, 2018.

\bibitem[Sun et~al.(2019)Sun, Deng, Nie, and Tang]{sun2019rotate}
Zhiqing Sun, Zhi-Hong Deng, Jian-Yun Nie, and Jian Tang.
\newblock Rotate: Knowledge graph embedding by relational rotation in complex
  space.
\newblock \emph{arXiv preprint arXiv:1902.10197}, 2019.

\bibitem[Sverrisson et~al.(2021)Sverrisson, Feydy, Correia, and
  Bronstein]{sverrisson2021fast}
Freyr Sverrisson, Jean Feydy, Bruno~E Correia, and Michael~M Bronstein.
\newblock Fast end-to-end learning on protein surfaces.
\newblock In \emph{IEEE/CVF Conference on Computer Vision and Pattern
  Recognition}, pp.\  15272--15281, 2021.

\bibitem[Tan \& Le(2019)Tan and Le]{tan2019efficientnet}
Mingxing Tan and Quoc Le.
\newblock Efficientnet: Rethinking model scaling for convolutional neural
  networks.
\newblock In \emph{International Conference on Machine Learning}, 2019.

\bibitem[Tan \& Le(2021)Tan and Le]{tan2021efficientnetv2}
Mingxing Tan and Quoc Le.
\newblock Efficientnetv2: Smaller models and faster training.
\newblock In \emph{International Conference on Machine Learning}, 2021.

\bibitem[Tolstikhin et~al.(2021)Tolstikhin, Houlsby, Kolesnikov, Beyer, Zhai,
  Unterthiner, Yung, Steiner, Keysers, Uszkoreit, et~al.]{tolstikhin2021mlp}
Ilya~O Tolstikhin, Neil Houlsby, Alexander Kolesnikov, Lucas Beyer, Xiaohua
  Zhai, Thomas Unterthiner, Jessica Yung, Andreas Steiner, Daniel Keysers,
  Jakob Uszkoreit, et~al.
\newblock Mlp-mixer: An all-mlp architecture for vision.
\newblock \emph{Advances in Neural Information Processing Systems}, 2021.

\bibitem[Toutanova \& Chen(2015)Toutanova and Chen]{toutanova2015observed}
Kristina Toutanova and Danqi Chen.
\newblock Observed versus latent features for knowledge base and text
  inference.
\newblock In \emph{Proceedings of the 3rd workshop on continuous vector space
  models and their compositionality}, pp.\  57--66, 2015.

\bibitem[Touvron et~al.(2021{\natexlab{a}})Touvron, Bojanowski, Caron, Cord,
  El-Nouby, Grave, Izacard, Joulin, Synnaeve, Verbeek,
  et~al.]{touvron2021resmlp}
Hugo Touvron, Piotr Bojanowski, Mathilde Caron, Matthieu Cord, Alaaeldin
  El-Nouby, Edouard Grave, Gautier Izacard, Armand Joulin, Gabriel Synnaeve,
  Jakob Verbeek, et~al.
\newblock Resmlp: Feedforward networks for image classification with
  data-efficient training.
\newblock \emph{arXiv preprint arXiv:2105.03404}, 2021{\natexlab{a}}.

\bibitem[Touvron et~al.(2021{\natexlab{b}})Touvron, Cord, Douze, Massa,
  Sablayrolles, and J{\'e}gou]{touvron2021training}
Hugo Touvron, Matthieu Cord, Matthijs Douze, Francisco Massa, Alexandre
  Sablayrolles, and Herv{\'e} J{\'e}gou.
\newblock Training data-efficient image transformers \& distillation through
  attention.
\newblock In \emph{International Conference on Machine Learning},
  2021{\natexlab{b}}.

\bibitem[Trouillon et~al.(2016)Trouillon, Welbl, Riedel, Gaussier, and
  Bouchard]{trouillon2016complex}
Th{\'e}o Trouillon, Johannes Welbl, Sebastian Riedel, {\'E}ric Gaussier, and
  Guillaume Bouchard.
\newblock Complex embeddings for simple link prediction.
\newblock In \emph{International Conference on Machine Learning}, 2016.

\bibitem[Vashishth et~al.(2019)Vashishth, Sanyal, Nitin, and
  Talukdar]{vashishth2019composition}
Shikhar Vashishth, Soumya Sanyal, Vikram Nitin, and Partha Talukdar.
\newblock Composition-based multi-relational graph convolutional networks.
\newblock \emph{arXiv preprint arXiv:1911.03082}, 2019.

\bibitem[Vaswani et~al.(2017)Vaswani, Shazeer, Parmar, Uszkoreit, Jones, Gomez,
  Kaiser, and Polosukhin]{vaswani2017attention}
Ashish Vaswani, Noam Shazeer, Niki Parmar, Jakob Uszkoreit, Llion Jones,
  Aidan~N Gomez, {\L}ukasz Kaiser, and Illia Polosukhin.
\newblock Attention is all you need.
\newblock \emph{Advances in Neural Information Processing Systems}, 2017.

\bibitem[Veli{\v{c}}kovi{\'c} et~al.(2017)Veli{\v{c}}kovi{\'c}, Cucurull,
  Casanova, Romero, Lio, and Bengio]{velivckovic2017graph}
Petar Veli{\v{c}}kovi{\'c}, Guillem Cucurull, Arantxa Casanova, Adriana Romero,
  Pietro Lio, and Yoshua Bengio.
\newblock Graph attention networks.
\newblock \emph{arXiv preprint arXiv:1710.10903}, 2017.

\bibitem[Wang et~al.(2021)Wang, Xie, Li, Fan, Song, Liang, Lu, Luo, and
  Shao]{wang2021pyramid}
Wenhai Wang, Enze Xie, Xiang Li, Deng-Ping Fan, Kaitao Song, Ding Liang, Tong
  Lu, Ping Luo, and Ling Shao.
\newblock Pyramid vision transformer: A versatile backbone for dense prediction
  without convolutions.
\newblock In \emph{IEEE/CVF International Conference on Computer Vision}, 2021.

\bibitem[Xiao et~al.(2018)Xiao, Liu, Zhou, Jiang, and Sun]{xiao2018unified}
Tete Xiao, Yingcheng Liu, Bolei Zhou, Yuning Jiang, and Jian Sun.
\newblock Unified perceptual parsing for scene understanding.
\newblock In \emph{European Conference on Computer Vision}, 2018.

\bibitem[Yang et~al.(2014)Yang, Yih, He, Gao, and Deng]{yang2014embedding}
Bishan Yang, Wen-tau Yih, Xiaodong He, Jianfeng Gao, and Li~Deng.
\newblock Embedding entities and relations for learning and inference in
  knowledge bases.
\newblock \emph{arXiv preprint arXiv:1412.6575}, 2014.

\bibitem[Yang et~al.(2021)Yang, Li, Zhang, Dai, Xiao, Yuan, and
  Gao]{yang2021focal}
Jianwei Yang, Chunyuan Li, Pengchuan Zhang, Xiyang Dai, Bin Xiao, Lu~Yuan, and
  Jianfeng Gao.
\newblock Focal attention for long-range interactions in vision transformers.
\newblock \emph{Advances in Neural Information Processing Systems}, 2021.

\bibitem[Yang et~al.(2022)Yang, Li, and Gao]{yang2022focal}
Jianwei Yang, Chunyuan Li, and Jianfeng Gao.
\newblock Focal modulation networks.
\newblock \emph{arXiv preprint arXiv:2203.11926}, 2022.

\bibitem[Yun et~al.(2019)Yun, Han, Oh, Chun, Choe, and Yoo]{yun2019cutmix}
Sangdoo Yun, Dongyoon Han, Seong~Joon Oh, Sanghyuk Chun, Junsuk Choe, and
  Youngjoon Yoo.
\newblock Cutmix: Regularization strategy to train strong classifiers with
  localizable features.
\newblock In \emph{IEEE/CVF International Conference on Computer Vision}, 2019.

\bibitem[Zhang et~al.(2017)Zhang, Cisse, Dauphin, and
  Lopez-Paz]{zhang2017mixup}
Hongyi Zhang, Moustapha Cisse, Yann~N Dauphin, and David Lopez-Paz.
\newblock mixup: Beyond empirical risk minimization.
\newblock \emph{arXiv preprint arXiv:1710.09412}, 2017.

\bibitem[Zhang et~al.(2019)Zhang, Li, Arnab, Yang, Tong, and
  Torr]{zhang2019dual}
Li~Zhang, Xiangtai Li, Anurag Arnab, Kuiyuan Yang, Yunhai Tong, and Philip~HS
  Torr.
\newblock Dual graph convolutional network for semantic segmentation.
\newblock \emph{arXiv preprint arXiv:1909.06121}, 2019.

\bibitem[Zhang et~al.(2020)Zhang, Xu, Arnab, and Torr]{zhang2020dynamic}
Li~Zhang, Dan Xu, Anurag Arnab, and Philip~HS Torr.
\newblock Dynamic graph message passing networks.
\newblock In \emph{IEEE/CVF Conference on Computer Vision and Pattern
  Recognition}, 2020.

\bibitem[Zhang et~al.(2022)Zhang, Xu, Jamasb, Chenthamarakshan, Lozano, Das,
  and Tang]{zhang2022protein}
Zuobai Zhang, Minghao Xu, Arian Jamasb, Vijil Chenthamarakshan, Aurelie Lozano,
  Payel Das, and Jian Tang.
\newblock Protein representation learning by geometric structure pretraining.
\newblock \emph{arXiv preprint arXiv:2203.06125}, 2022.

\bibitem[Zhou et~al.(2017)Zhou, Zhao, Puig, Fidler, Barriuso, and
  Torralba]{zhou2017scene}
Bolei Zhou, Hang Zhao, Xavier Puig, Sanja Fidler, Adela Barriuso, and Antonio
  Torralba.
\newblock Scene parsing through ade20k dataset.
\newblock In \emph{IEEE/CVF Conference on Computer Vision and Pattern
  Recognition}, 2017.

\bibitem[Zhu et~al.(2021)Zhu, Zhang, Xhonneux, and Tang]{zhu2021neural}
Zhaocheng Zhu, Zuobai Zhang, Louis-Pascal Xhonneux, and Jian Tang.
\newblock Neural bellman-ford networks: A general graph neural network
  framework for link prediction.
\newblock \emph{Advances in Neural Information Processing Systems}, 2021.

\bibitem[Zhu et~al.(2022)Zhu, Shi, Zhang, Liu, Xu, Yuan, Zhang, Chen, Cai, Lu,
  et~al.]{zhu2022torchdrug}
Zhaocheng Zhu, Chence Shi, Zuobai Zhang, Shengchao Liu, Minghao Xu, Xinyu Yuan,
  Yangtian Zhang, Junkun Chen, Huiyu Cai, Jiarui Lu, et~al.
\newblock Torchdrug: A powerful and flexible machine learning platform for drug
  discovery.
\newblock \emph{arXiv preprint arXiv:2202.08320}, 2022.

\end{thebibliography}
\bibliographystyle{iclr2023_conference}


\newpage
\appendix

\section{FLOPs of RGConv and GRMP} \label{supp:sec:flops}

For FLOPs computation, we consider the multi-relational graph $\mathcal{G} = (\mathcal{V}, \mathcal{E}, \mathcal{R})$ with node set $\mathcal{V}$, edge set $\mathcal{E}$ and relation (\emph{i.e.}, edge type) set $\mathcal{R}$, and both input and output node features are with $C$ feature channels. In addition, we assume that, when introducing a new relation, the in-degree of each node will increase by $\bar{d}$ on average. 

\begin{myprop} \label{prop:rgconv_flops}
To process the assumed multi-relational graph, the Relational Graph Convolution (RGConv) consumes the FLOPs as below under the efficient implementation with sparse matrix multiplication:
\begin{displaymath}
\mathrm{FLOPs}(\mathrm{RGConv}) = |\mathcal{R}| \cdot (2 \bar{d} |\mathcal{V}| C + 2 |\mathcal{V}| C^2) + 2 |\mathcal{V}| C^2 + |\mathcal{V}| C .
\end{displaymath}
\end{myprop}
\begin{proof}
We divide the computation of RGConv into three steps and compute the FLOPs of each step:
\begin{enumerate}[label=\protect\circled{\arabic*}]
\item In the first step, the adjacency of all node pairs on $|\mathcal{R}|$ different relations are summarized in the adjacency matrix $A \in \mathbb{R}^{|\mathcal{V}| \times |\mathcal{R}||\mathcal{V}|}$, where the element $A_{i, (j-1)|\mathcal{R}| + k}$ indicates the weight of the edge from the $i$-th node to the $j$-th node with the $k$-th relation:
\begin{equation} \label{supp:eq:adjacency}
A_{i, (j-1)|\mathcal{R}| + k} = 
    \begin{cases}
      \frac{1}{|\mathcal{N}_{r_k}(v_j)|} & \text{there is an edge from $i$-th node to $j$-th node with $k$-th relation,}\\
      0 & \text{otherwise,}
    \end{cases}  
\end{equation}
where $\mathcal{N}_{r_k}(v_j) = \{ u | (u, v_j, r_k) \in \mathcal{E} \}$ is the neighborhood set of node $v_j$ with relation $r_k$. Using this adjacency matrix, each node will have $|\mathcal{R}|$ different slots to receive the relational messages passed to it. All relational message passing operations can be realized by a sparse matrix multiplication:
\begin{equation} \label{supp:eq:rgconv_step1}
\tilde{Z} = A^\top Z ,
\end{equation}
where $Z \in \mathbb{R}^{|\mathcal{V}| \times C}$ denotes input node features, and $\tilde{Z} \in \mathbb{R}^{|\mathcal{R}||\mathcal{V}| \times C}$ denotes the relational slots of all nodes after message passing. By utilizing the sparsity of the adjacency matrix, this step consumes following FLOPs:
\begin{equation} \label{supp:eq:rgconv_step1_flops}
\mathrm{FLOPs}(\mathrm{RGConv}\!-\!\mathrm{{\small \circled{1}}}) = 2 |\mathcal{E}| C = 2 \bar{d} |\mathcal{R}| |\mathcal{V}| C .
\end{equation}

\item In the second step, we first integrate the relational slots of each node to get the reshaped $\tilde{Z} \in \mathbb{R}^{|\mathcal{V}| \times |\mathcal{R}|C}$. At this time, each node is represented by a $|\mathcal{R}|C$-dimensional vector, \emph{i.e.}, the aggregated messages of all relations. Next, we concatenate the convolutional kernel matrices of all relations to produce $W^{\mathrm{conv}} \in \mathbb{R}^{|\mathcal{R}|C \times C}$, and this matrix is applied upon $\tilde{Z}$ to combine the messages in the same relational slot and aggregate messages across different relations:
\begin{equation} \label{supp:eq:rgconv_step2}
Z^{\mathrm{aggr}} = \tilde{Z} W^{\mathrm{conv}} ,
\end{equation}
where $Z^{\mathrm{aggr}} \in \mathbb{R}^{|\mathcal{V}| \times C}$ denotes the aggregated neighborhood information for each node. This step has the FLOPs as below:
\begin{equation} \label{supp:eq:rgconv_step2_flops}
\mathrm{FLOPs}(\mathrm{RGConv}\!-\!\mathrm{{\small \circled{2}}}) = 2 |\mathcal{R}| |\mathcal{V}| C^2 .
\end{equation}

\item In the final step, a self-update with matrix $W^{\mathrm{self}} \in \mathbb{R}^{C \times C}$ is first performed on the input feature of each node, and the self-updated node feature is further added with the aggregated neighborhood information:
\begin{equation} \label{supp:eq:rgconv_step3}
Z' = Z W^{\mathrm{self}} + Z^{\mathrm{aggr}} ,
\end{equation}
where $Z' \in \mathbb{R}^{|\mathcal{V}| \times C}$ denotes output node features. This step has the FLOPs as below:
\begin{equation} \label{supp:eq:rgconv_step3_flops}
\mathrm{FLOPs}(\mathrm{RGConv}\!-\!\mathrm{{\small \circled{3}}}) = 2 |\mathcal{V}| C^2 + |\mathcal{V}| C .
\end{equation}
\end{enumerate}

Therefore, by summing up the computational cost of three steps, the RGConv consumes the following FLOPs in total:
\begin{displaymath}
\mathrm{FLOPs}(\mathrm{RGConv}) = |\mathcal{R}| \cdot (2 \bar{d} |\mathcal{V}| C + 2 |\mathcal{V}| C^2) + 2 |\mathcal{V}| C^2 + |\mathcal{V}| C .
\end{displaymath} 

\end{proof}

\begin{myprop} \label{prop:grmp_flops}
To process the assumed multi-relational graph, the Gated Relational Message Passing (GRMP) consumes the FLOPs as below under the efficient implementation with sparse matrix multiplication:
\begin{displaymath}
\mathrm{FLOPs}(\mathrm{GRMP}) = |\mathcal{R}| \cdot (2 \bar{d} + 7) |\mathcal{V}| C  + 6|\mathcal{V}| C^2 .
\end{displaymath}
\end{myprop}
\begin{proof}
Following the steps of GRMP stated in Eq.~\eqref{eq:grmp}, we compute the FLOPs of each step:
\begin{enumerate}[label=\protect\circled{\arabic*}]
\item In the first step, we conduct a pre-layer node-wise channel aggregation with the weight matrix $W^{\mathrm{in}} \in \mathbb{R}^{C \times C}$:
\begin{equation} \label{supp:eq:grmp_step1}
Z^{\mathrm{in}} = Z W^{\mathrm{in}} , 
\end{equation}
where $Z \in \mathbb{R}^{|\mathcal{V}| \times C}$ denotes the input node features, and $Z^{\mathrm{in}} \in \mathbb{R}^{|\mathcal{V}| \times C}$ denotes the channel-aggregated node features. This step has the FLOPs consumption as below:
\begin{equation} \label{supp:eq:grmp_step1_flops}
\mathrm{FLOPs}(\mathrm{GRMP}\!-\!\mathrm{{\small \circled{1}}}) = 2 |\mathcal{V}| C^2 .
\end{equation}

\item In the second step, we first gather the messages within the same relation for each node, which is realized by the sparse matrix multiplication between $Z^{\mathrm{in}}$ and the adjacency matrix $A \in \mathbb{R}^{|\mathcal{V}| \times |\mathcal{R}||\mathcal{V}|}$ ($A$ is identically defined as in the step {\small \circled{1}} of Proposition~\ref{prop:rgconv_flops}):
\begin{equation} \label{supp:eq:grmp_step2_aggr}
\tilde{Z}^{\mathrm{in}} = A^\top Z^{\mathrm{in}} ,
\end{equation}
where $\tilde{Z}^{\mathrm{in}} \in \mathbb{R}^{|\mathcal{R}||\mathcal{V}| \times C}$ represents the relational slots of all nodes after message passing. The relational slots of each node are then integrated to get the reshaped $\tilde{Z}^{\mathrm{in}} \in \mathbb{R}^{|\mathcal{V}| \times |\mathcal{R}|C}$. By concatenating the convolutional kernel vectors of all relations, we have $w_{\mathrm{conv}} \in \mathbb{R}^{|\mathcal{R}|C \times 1}$, and this vector is broadcast to all nodes to perform channel-wise message aggregation via Hadamard product:
\begin{equation} \label{supp:eq:grmp_step2_conv}
\tilde{Z}^{\mathrm{aggr}} = (\mathbf{1}_{\mathrm{conv}} w_\mathrm{conv}^\top) \odot \tilde{Z}^{\mathrm{in}} ,
\end{equation}
where $\mathbf{1}_{\mathrm{conv}} \in \mathbb{R}^{|\mathcal{V}| \times 1}$ is the all-one vector for broadcasting, and $\tilde{Z}^{\mathrm{aggr}} \in \mathbb{R}^{|\mathcal{V}| \times |\mathcal{R}|C}$ denotes the relational slots of all nodes after intra-relation message aggregation.  

To conduct the operations in Eqs.~\eqref{supp:eq:grmp_step2_aggr} and \eqref{supp:eq:grmp_step2_conv}, this step consumes the following FLOPs:
\begin{equation} \label{supp:eq:grmp_step2_flops}
\mathrm{FLOPs}(\mathrm{GRMP}\!-\!\mathrm{{\small \circled{2}}}) = 2 |\mathcal{E}| C + 2 |\mathcal{R}| |\mathcal{V}| C =  2 \bar{d} |\mathcal{R}| |\mathcal{V}| C + 2 |\mathcal{R}| |\mathcal{V}| C .  
\end{equation}

\item In the third step, we first compute the attentive weights assigned to all relations on each node:
\begin{equation} \label{supp:eq:grmp_step3_weight}
M^{\alpha} = Z W^{\alpha} ,
\end{equation}
where $W^{\alpha} \in \mathbb{R}^{C \times |\mathcal{R}|}$ is the weight matrix for node-adaptive relation weighting, and $M^{\alpha} \in \mathbb{R}^{|\mathcal{V}| \times |\mathcal{R}|}$ denotes the relation weights on all nodes. After that, a weighted summation is performed to aggregate the messages of different relations in $\tilde{Z}^{\mathrm{aggr}}$ (in this operation, we use the reshaped $\tilde{Z}^{\mathrm{aggr}} \in \mathbb{R}^{|\mathcal{V}| \times |\mathcal{R}| \times C}$ and the reshaped $M^{\alpha} \in \mathbb{R}^{|\mathcal{V}| \times |\mathcal{R}| \times 1}$): 
\begin{equation} \label{supp:eq:grmp_step3_sum}
\arc{Z}^{\mathrm{aggr}} = \sum_{i=1}^{|\mathcal{R}|} (M^{\alpha}_{:,i,:} \, \mathbf{1}_\alpha^\top) \odot \tilde{Z}^{\mathrm{aggr}}_{:,i,:} ,
\end{equation}
where $\mathbf{1}_\alpha \in \mathbb{R}^{C \times 1}$ is the all-one vector for broadcasting relation weights to all feature channels, and $\arc{Z}^{\mathrm{aggr}} \in \mathbb{R}^{|\mathcal{V}| \times C}$ denotes the per-node neighborhood representations after inter-relation message aggregation.   

To perform Eqs.~\eqref{supp:eq:grmp_step3_weight} and \eqref{supp:eq:grmp_step3_sum}, this step has the following FLOPs consumption:
\begin{equation} \label{supp:eq:grmp_step3_flops}
\mathrm{FLOPs}(\mathrm{GRMP}\!-\!\mathrm{{\small \circled{3}}}) = 2 |\mathcal{R}| |\mathcal{V}| C + |\mathcal{R}| \cdot 2|\mathcal{V}| C + (|\mathcal{R}|-1) |\mathcal{V}| C = 5 |\mathcal{R}| |\mathcal{V}| C - |\mathcal{V}| C . 
\end{equation}

\item The fourth step conducts a post-layer node-wise channel aggregation with the weight matrix $W^{\mathrm{out}} \in \mathbb{R}^{C \times C}$:
\begin{equation} \label{supp:eq:grmp_step4}
Z^{\mathrm{aggr}} = \arc{Z}^{\mathrm{aggr}} W^{\mathrm{out}} , 
\end{equation}
where $Z^{\mathrm{aggr}} \in \mathbb{R}^{|\mathcal{V}| \times C}$ denotes the channel-aggregated neighborhood representations. This step consumes the FLOPs as below:
\begin{equation} \label{supp:eq:grmp_step4_flops}
\mathrm{FLOPs}(\mathrm{GRMP}\!-\!\mathrm{{\small \circled{4}}}) = 2 |\mathcal{V}| C^2 .
\end{equation}

\item In the final step, the input feature of each node first performs self-update with the weight matrix $W^{\mathrm{self}} \in \mathbb{R}^{C \times C}$, and the self-updated node feature is further updated by its neighborhood representation via a gating mechanism:
\begin{equation} \label{supp:eq:grmp_step5}
Z' = Z W^{\mathrm{self}} \odot Z^{\mathrm{aggr}} ,
\end{equation}
where $Z' \in \mathbb{R}^{|\mathcal{V}| \times C}$ denotes output node features. This step has the FLOPs as below:
\begin{equation} \label{supp:eq:grmp_step5_flops}
\mathrm{FLOPs}(\mathrm{GRMP}\!-\!\mathrm{{\small \circled{5}}}) = 2 |\mathcal{V}| C^2 + |\mathcal{V}| C .
\end{equation}
\end{enumerate}

Therefore, by summing up the computational cost of five steps, the GRMP has the following FLOPs consumption in total:
\begin{equation} \label{supp:eq:rgconv_flops}
\mathrm{FLOPs}(\mathrm{GRMP}) = |\mathcal{R}| \cdot (2 \bar{d} + 7) |\mathcal{V}| C  + 6|\mathcal{V}| C^2 .
\end{equation}

\end{proof}


\section{Graphical Illustration of GRMP} \label{supp:sec:grmp_fig}

\vspace{-1.8mm}
\begin{figure}[h]
\centering
    \includegraphics[width=1.0\linewidth]{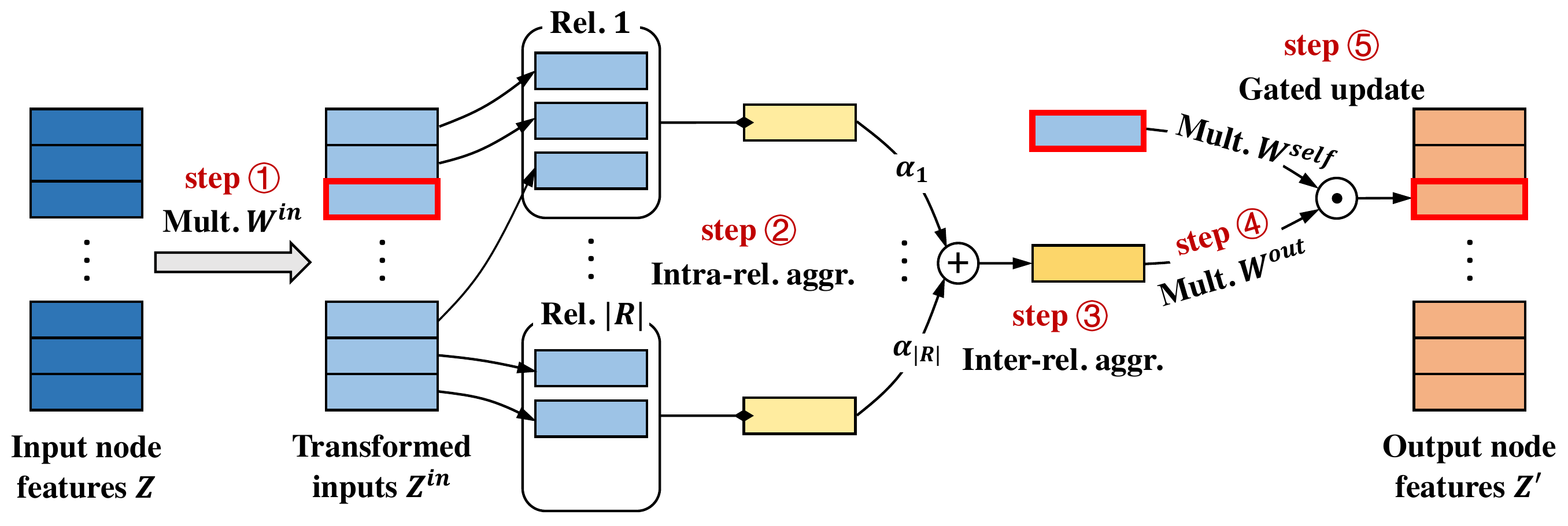}
    \vspace{-6.0mm}
    \caption{Graphical illustration for node representation update in the GRMP layer. We specifically show the neighborhood aggregation and representation update procedure of the node denoted in \textcolor{red}{red}. \emph{Abbr.}, Multi.: multiply with; Rel.: relation; aggr.: aggregation.}
    \label{supp:fig:grmp}
\end{figure}
\vspace{0.3mm}


In Fig.~\ref{supp:fig:grmp}, we graphically illustrate the mechanism of node representation update in the GRMP layer. In specific, GRMP updates the node representation matrix from $Z$ to $Z'$ with the following steps: 
\begin{enumerate}[label=\protect\circled{\arabic*}]
    \item A linear layer transforms the input node representations $Z \in \mathbb{R}^{|\mathcal{V}| \times C}$ to $Z^{\mathrm{in}} \in \mathbb{R}^{|\mathcal{V}| \times C}$, which aggregates the feature channels of each node at the beginning of the layer.
    \item For each node, its neighbors are assigned to different groups according to their relations with the node, and the neighbors in each group are aggregated in a channel-wise way.  
    \item The aggregated messages of different relational groups are then scaled by per-relation scalar weights $\{\alpha_r\}_{r=1}^{|\mathcal{R}|}$ and aggregated to the neighborhood representations $\arc{Z}^{\mathrm{aggr}} \in \mathbb{R}^{|\mathcal{V}| \times C}$.
    \item $\arc{Z}^{\mathrm{aggr}}$ is then transformed by a linear layer to aggregate the feature channels of each node's neighbors, deriving the transformed neighborhood representations $Z^{\mathrm{aggr}} \in \mathbb{R}^{|\mathcal{V}| \times C}$.
    \item Finally, $Z^{\mathrm{aggr}}$ serves as the gate to update all node representations, deriving the output node representations $Z' \in \mathbb{R}^{|\mathcal{V}| \times C}$.
\end{enumerate}


\section{Detailed Model Architecture for Image Modeling} \label{supp:sec:arch}


\begin{table}[h]
\begin{spacing}{1.1}
    \vspace{-4mm}
    \centering
    \caption{Detailed architectures of {\model}-T/S/B for ImageNet-1K classification (\#parameters and FLOPs are computed under the resolution $224\times224$). $H \times W$: input image resolution; $C$: number of feature channels; $\gamma$: FFN's hidden dimension ratio; $K$: number of K-nearest neighbors for medium-range edges; $\mathcal{Y}$: label set for classification. ``T'' denotes the tiny model; ``S'' denotes the small model; ``B'' denotes the base model.}
    \vspace{0.5mm}
    \label{supp:tab:cv_arch}
    \begin{adjustbox}{max width=1\linewidth}
        \begin{tabular}{c|c|c|c|c}
            \toprule
            \bf{Module} & \bf{\#Patches} & \bf{{\model}-T} & \bf{{\model}-S} & \bf{{\model}-B} \\
            \midrule \midrule
            \bf{Stem} & $\frac{H}{4}\!\times\!\frac{W}{4}$ & $4\!\times\!4$ conv, stride$\;\!$=$\;\!$4 & $4\!\times\!4$ conv, stride$\;\!$=$\;\!$4 & $4\!\times\!4$ conv, stride$\;\!$=$\;\!$4 \\
            \midrule 
            \splitcell{\bf{Graph} \\ \bf{Construction}} & $\frac{H}{4}\!\times\!\frac{W}{4}$ & \bsplitcell{short-range edges, \\ long-range edges} & \bsplitcell{short-range edges, \\ long-range edges} & \bsplitcell{short-range edges, \\ long-range edges} \\
            \midrule
            \bf{Stage 1} & $\frac{H}{4}\!\times\!\frac{W}{4}$ & \bsplitcell{GRMP ($C\!=\!96$), \\ FFN ($C\!=\!96$, $\gamma\!=\!4$)}$\times 2$ & \bsplitcell{GRMP ($C\!=\!96$), \\ FFN ($C\!=\!96$, $\gamma\!=\!4$)}$\times 2$ & \bsplitcell{GRMP ($C\!=\!128$), \\ FFN ($C\!=\!128$, $\gamma\!=\!4$)}$\times 2$ \\
            \midrule
            \bf{Downsample} & $\frac{H}{8}\!\times\!\frac{W}{8}$ & PatchMerging & PatchMerging & PatchMerging \\
            \midrule
            \splitcell{\bf{Graph} \\ \bf{Construction}} & $\frac{H}{8}\!\times\!\frac{W}{8}$ & \bsplitcell{short-range edges, \\ medium-range edges ($K\!=\!12$), \\ long-range edges} & \bsplitcell{short-range edges, \\ medium-range edges ($K\!=\!12$), \\ long-range edges} & \bsplitcell{short-range edges, \\ medium-range edges ($K\!=\!12$), \\ long-range edges} \\
            \midrule
            \bf{Stage 2} & $\frac{H}{8}\!\times\!\frac{W}{8}$ & \bsplitcell{GRMP ($C\!=\!192$), \\ FFN ($C\!=\!192$, $\gamma\!=\!4$)}$\times 2$ & \bsplitcell{GRMP ($C\!=\!192$), \\ FFN ($C\!=\!192$, $\gamma\!=\!4$)}$\times 2$ & \bsplitcell{GRMP ($C\!=\!256$), \\ FFN ($C\!=\!256$, $\gamma\!=\!4$)}$\times 2$ \\
            \midrule
            \bf{Downsample} & $\frac{H}{16}\!\times\!\frac{W}{16}$ & PatchMerging & PatchMerging & PatchMerging \\
            \midrule
            \splitcell{\bf{Graph} \\ \bf{Construction}} & $\frac{H}{16}\!\times\!\frac{W}{16}$ & \bsplitcell{short-range edges, \\ medium-range edges ($K\!=\!12$), \\ long-range edges} & \bsplitcell{short-range edges, \\ medium-range edges ($K\!=\!12$), \\ long-range edges} & \bsplitcell{short-range edges, \\ medium-range edges ($K\!=\!12$), \\ long-range edges} \\
            \midrule
            \bf{Stage 3} & $\frac{H}{16}\!\times\!\frac{W}{16}$ & \bsplitcell{GRMP ($C\!=\!384$), \\ FFN ($C\!=\!384$, $\gamma\!=\!4$)}$\times 6$ & \bsplitcell{GRMP ($C\!=\!384$), \\ FFN ($C\!=\!384$, $\gamma\!=\!4$)}$\times 18$ & \bsplitcell{GRMP ($C\!=\!512$), \\ FFN ($C\!=\!512$, $\gamma\!=\!4$)}$\times 18$ \\
            \midrule
            \bf{Downsample} & $\frac{H}{32}\!\times\!\frac{W}{32}$ & PatchMerging & PatchMerging & PatchMerging \\
            \midrule
            \splitcell{\bf{Graph} \\ \bf{Construction}} & $\frac{H}{32}\!\times\!\frac{W}{32}$ & \bsplitcell{short-range edges, \\ medium-range edges ($K\!=\!12$), \\ long-range edges} & \bsplitcell{short-range edges, \\ medium-range edges ($K\!=\!12$), \\ long-range edges} & \bsplitcell{short-range edges, \\ medium-range edges ($K\!=\!12$), \\ long-range edges} \\
            \midrule
            \bf{Stage 4} & $\frac{H}{32}\!\times\!\frac{W}{32}$ & \bsplitcell{GRMP ($C\!=\!768$), \\ FFN ($C\!=\!768$, $\gamma\!=\!4$)}$\times 2$ & \bsplitcell{GRMP ($C\!=\!768$), \\ FFN ($C\!=\!768$, $\gamma\!=\!4$)}$\times 2$ & \bsplitcell{GRMP ($C\!=\!1024$), \\ FFN ($C\!=\!1024$, $\gamma\!=\!4$)}$\times 2$ \\
            \midrule
            \bf{Head} & $1\!\times\!1$ & Pooling \& Linear ($|\mathcal{Y}|\!=\!1000$)  & Pooling \& Linear ($|\mathcal{Y}|\!=\!1000$) & Pooling \& Linear ($|\mathcal{Y}|\!=\!1000$) \\
            \midrule
            \multicolumn{2}{c|}{\bf{\#Parameters (M)}} & 28.8 & 50.2 & 88.7 \\
            \midrule
            \multicolumn{2}{c|}{\bf{FLOPs (G)}} & 4.6 & 8.8 & 15.6 \\
            \bottomrule
        \end{tabular}
    \end{adjustbox}
    \vspace{-0.5mm}
\end{spacing}
\end{table}


For image modeling, we basically follow the hierarchical architecture proposed by Swin Transformer~\citep{liu2021swin}, as summarized in Tab.~\ref{supp:tab:cv_arch}. The architecture begins with a patch embedding module implemented by non-overlapping 2D convolution. After that, the model is split into 4 modeling stages: (1) the number of patches (\emph{i.e.}, nodes in our graph) is reduced to a quarter across consecutive stages by the ``PatchMerging'' operation~\citep{liu2021swin}; (2) increasing feature channels $[C, 2C, 4C, 8C]$ are used for all stages. We place a graph construction layer before each modeling stage to update the multi-relational graph structure. For the first stage, we only use short- and long-range edges to reduce the computational cost (computing medium-range edges by representation similarity comparison is expensive in the first stage with many patches), and the relational edges of all three ranges are adopted in the last three stages. Each stage is composed of multiple modeling blocks, where each block contains a GRMP layer (Sec.~\ref{sec:method:layer}) for relational message passing and a feed-forward network (FFN)~\citep{vaswani2017attention} for feature transformation. In the end, a global average pooling layer produces the whole-image representation, and a linear head outputs the final prediction. We adjust the number of feature channels and the number of blocks in each stage to derive {\model}-T, {\model}-S and {\model}-B with standard number of parameters and FLOPs. We implement the models based on the PyTorch~\citep{pytorch} deep learning library. 


\section{Introduction to Protein Structure} \label{supp:sec:protein_intro}

Proteins are macromolecules that perform critical biological functions in living organisms. A protein owns multiple levels of structures, as described below:
\vspace{-1mm}
\begin{itemize}
    \item \textbf{Primary structure} (Fig.~\ref{supp:fig:protein_structure}(a)). At the chemical level, a protein is composed of one or multiple chains of amino acid residues, forming the protein sequence which is the primary protein structure. In the protein sequence $s=(s_1, s_2, \cdots, s_L)$, each element $s_l$ denotes a type of amino acid (there are 20 common amino acids and two rare ones, \emph{i.e.}, Selenocysteine and Pyrrolysine). The primary structure tells the sequential order of amino acids in a protein, but otherwise it does not reveal any information about the 3D folded structure of the protein. This fact limits its usefulness in the analysis/prediction of protein functions, due to the principle that ``protein folded structures largely determine their functions''~\citep{harms2010analyzing}.
    \item \textbf{Secondary structure} (Fig.~\ref{supp:fig:protein_structure}(b)). The secondary structures of proteins are some repeatedly-occurred local structures like the $\alpha$-helices shown in Fig.~\ref{supp:fig:protein_structure}(b). These structures are stabilized by hydrogen bonds, and, together with the tight turns and flexible loops in between, they constitute the complete protein folded structure.  
    \item \textbf{Tertiary structure} (Fig.~\ref{supp:fig:protein_structure}(c)). The spatial arrangement of different secondary structure components leads to the formation of the tertiary structure (\emph{i.e.}, the folded structure of a protein). The tertiary structure is jointly held by short-range interactions like hydrogen bonding and long-range interactions like hydrophobic interactions. Thanks to the recent advances of highly accurate protein folded structure predictors based on deep learning~\citep{jumper2021highly,baek2021accurate}, we can now efficiently acquire numerous previously unknown protein tertiary structures with reasonable confidence. These advances are expected to promote the understanding of protein functions based on tertiary structures. 
\end{itemize}


\vspace{-1.0mm}
\begin{figure}[t]
\centering
    \includegraphics[width=1.0\linewidth]{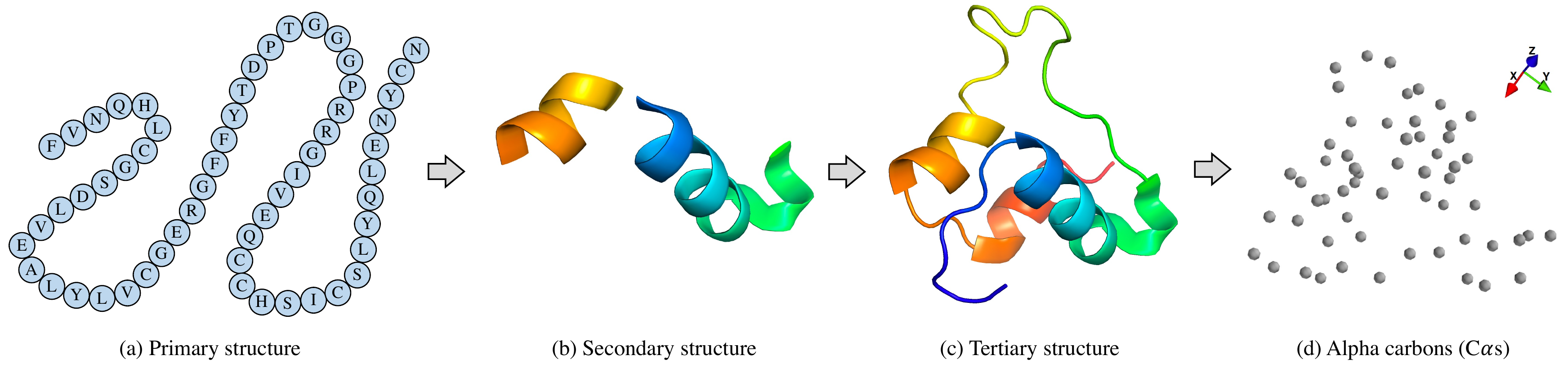}
    \vspace{-6.5mm}
    \caption{The primary structure, secondary structure, tertiary structure and all alpha carbons of the single-chain insulin protein (ID in PDB~\citep{berman2000protein}: 2LWZ).}
    \label{supp:fig:protein_structure}
    \vspace{-0.5mm}
\end{figure}


In this work, we focus on protein function prediction tasks based on tertiary structures. Specifically, we adopt an informative and light-weight representation format, \emph{i.e.}, \textbf{all alpha carbons (C$\bm{\alpha}$s) in the tertiary structure} (Fig.~\ref{supp:fig:protein_structure}(d)), which is widely used in the literature~\citep{gligorijevic2021structure,baldassarre2021graphqa,zhang2022protein}. A C$\alpha$ can be seen as the center of its corresponding amino acid, and thus the overall tertiary structure of a protein can be well captured by the collection of all C$\alpha$s. At this time, the C$\alpha$s are actually a set of separate points in the 3D space, since there is no chemical bond among them. To better describe the interactions within a protein, we seek to construct edges among C$\alpha$s and lead to a more informative representation format, \emph{i.e.}, the \textbf{C$\bm{\alpha}$ graph}.


\section{More Experimental Setups} \label{supp:sec:setup}

\subsection{More Experimental Setups on ImageNet-1K Classification} \label{supp:sec:setup:cls}

In the following, we state the detailed model and training configurations of (1) training on ImageNet-1K from scratch and (2) pretraining on ImageNet-22K followed by ImageNet-1K fine-tuning. For training configurations, we mainly follow the standards set up by Swin Transformer~\citep{liu2021swin} for fair comparison. 

\subsubsection{From-scratch Training on ImageNet-1K} \label{supp:sec:setup:cls:scratch}

\textbf{Model configurations.} The whole model architectures of {\model}-T, {\model}-S and {\model}-B are presented in Tab.~\ref{supp:tab:cv_arch}. For medium-range edges, 12 nearest semantic neighbors of each patch are linked to it to capture medium-range interactions. For long-range edges, we compute the representations of per-patch global-context virtual nodes by a stack of depth-wise 2D convolutions with the accumulative receptive field as 7, and these virtual nodes are linked to their corresponding patches. 

\textbf{Training configuration.} An AdamW~\citep{loshchilov2017decoupled} optimizer (betas: [$0.9$, $0.999$], weight decay: $0.05$) is employed to train each {\model} model for 300 epochs. We set the batch size as 2048, the base learning rate as $0.002$ and the gradient clipping norm as $5.0$. A cosine learning rate scheduler is adopted to adjust the learning rate from $2.0 \times 10^{-6}$ to $0.002$ in the first 20 warm-up epochs, and the learning rate is decayed to $2.0 \times 10^{-5}$ in the rest epochs with a cosine rate. The stochastic depth drop rates are set to $0.15$, $0.3$ and $0.5$ respectively for {\model}-T, {\model}-S and {\model}-B. We follow the augmentation functions and mixup strategies used in Swin Transformer. All experiments are conducted on 16 Tesla-V100-32GB GPUs.


\subsubsection{Pre-training on ImageNet-22K and Fine-tuning on ImageNet-1K} \label{supp:sec:setup:cls:pretrain}

\textbf{Model configurations.} The {\model}-B with the standard model architecture as in Tab.~\ref{supp:tab:cv_arch} is used, except that the last linear classification head outputs 21,841-dimensional logits to perform ImageNet-22K classification. 

\textbf{Training configuration.} For ImageNet-22K pre-training, we train {\model}-B with an AdamW optimizer (betas: [$0.9$, $0.999$], weight decay: $0.05$) for 90 epochs with the batch size 4096 and the image resolution $224 \times 224$. A cosine learning rate scheduler is employed to linearly increase the learning rate from $0$ to $4.0 \times 10^{-3}$ in the first 5 warm-up epochs, and it decays the learning rate to $1.0 \times 10^{-6}$ in the rest epochs with a cosine rate. The stochastic depth drop rate is set as $0.1$. All augmentation functions and mixup strategies follow Swin Transformer. The pre-training is performed on 64 Tesla-V100-32GB GPUs.

For ImageNet-1K fine-tuning, the pre-trained model is fine-tuned for 30 epochs by an AdamW optimizer (betas: [$0.9$, $0.999$], weight decay: $1.0 \times 10^{-8}$). The cosine learning rate scheduler adjusts the learning rate from $8.0 \times 10^{-8}$ to $8.0 \times 10^{-5}$ in the first 5 warm-up epochs, and the learning rate is decayed to $8.0 \times 10^{-7}$ in the rest epochs with a cosine rate. The stochastic depth drop rate is set as $0.2$. Both Mixup~\citep{zhang2017mixup} and CutMix~\citep{yun2019cutmix} are muted during fine-tuning, following FocalNet~\cite{yang2022focal}. The fine-tuning is performed on 16 Tesla-V100-32GB GPUs. 


\subsubsection{Throughput Computation} \label{supp:sec:setup:cls:throughput}

We follow Swin Transformer to measure the inference throughput on a Tesla-V100-32GB GPU with batch size 128. We adopt graph checkpoints to enhance the speed of inferring an image that has been seen. During inference, we add short-range edges to the list of medium-range edges and merge their corresponding relations to further promote the efficiency. 


\subsection{More Experimental Setups on COCO Object Detection} \label{supp:sec:setup:det}

\textbf{Model configurations.} We use the {\model}-T, {\model}-S and {\model}-B pre-trained on ImageNet-1K as the backbone of Mask R-CNN~\citep{he2017mask}. In specific, we take the patch representations output from all four modeling stages as the inputs of the Feature Pyramid Network (FPN)~\cite{fpn}. For medium-range edge construction on the high-resolution images of COCO, we select the semantic neighbors of each patch from a $112 \times 112$ dilated window (dilation ratio: 2) to reduce the computational cost. For long-range edge construction, the representations of per-patch global-context virtual nodes are computed by a stack of depth-wise 2D convolutions with the accumulative receptive field as 31, and these virtual nodes are linked to their corresponding patches. 

\textbf{Training configurations.} We follow Swin Transformer~\citep{liu2021swin} to adopt a multi-scale training strategy where the shorter side of an image is resized to $[480, 800]$, and the longer side is with length 1,333. An AdamW~\citep{loshchilov2017decoupled} optimizer (betas: [$0.9$, $0.999$], weight decay: $0.05$) with initial learning rate $1.0 \times 10^{-4}$ is employed for model training. In the 1{\texttimes} schedule with 12 total epochs, the learning rate is decayed at the 9th and 11th epoch with the decay rate 0.1. In the 3{\texttimes} schedule with 36 total epochs, the learning rate is decayed at the 27th and 33rd epoch with the decay rate 0.1. The stochastic depth drop rate is set as 0.1, 0.2, 0.3 in 1{\texttimes} schedule and 0.25, 0.5, 0.5 in 3{\texttimes} schedule for {\model}-T/S/B, respectively. All models are trained with batch size 8 on 8 Tesla-V100-32GB GPUs (\emph{i.e.}, one image per GPU). Our implementations are based on the mmdetection~\citep{mmdetection} framework.


\subsection{More Experimental Setups on ADE20K Semantic Segmentation} \label{supp:sec:setup:seg}

\textbf{Model configurations.} The {\model}-T, {\model}-S and {\model}-B pre-trained on ImageNet-1K serve as the backbone of UperNet~\citep{xiao2018unified} to perform semantic segmentation. The patch representations output by all four modeling stages serve as the inputs of the Feature Pyramid Network (FPN)~\cite{fpn}. For medium-range edge construction, each patch is connected with its semantic neighbors from a $144 \times 144$ dilated window (dilation ratio: 2). For long-range edge construction, we use a stack of depth-wise 2D convolutions with accumulative receptive field 31 to compute the representations of per-patch global-context virtual nodes, and we connect these virtual nodes with their corresponding patches. 

\textbf{Training configurations.} All input images are resized to the resolution $512 \times 512$. We adopt an AdamW~\citep{loshchilov2017decoupled} optimizer (betas: [$0.9$, $0.999$], weight decay: $0.01$) to train the model for 160K iterations with the base learning rate $6.0 \times 10^{-5}$. All models are trained with batch size 16 on 8 Tesla-V100-32GB GPUs (\emph{i.e.}, two images per GPU). Our implementations are based on the mmsegmentation~\citep{mmsegmentation} framework. 


\subsection{More Experimental Setups on Protein Function Prediction} \label{supp:sec:setup:protein}

\textbf{Edge message passing.} \citet{zhang2022protein} proposes to enhance the GearNet by edge-level message passing, which well captures the interactions between edges. To compare with the GearNet-Edge model enhanced in this way, we adapt the same edge message passing scheme to our {\model}. 

Specifically, based on the constructed multi-relational graph $\mathcal{G} = (\mathcal{V}, \mathcal{E}, \mathcal{R})$, we further construct a \emph{line graph}~\citep{harary1960some} $\mathcal{G}_{\mathrm{line}} = (\mathcal{V}_{\mathrm{line}}, \mathcal{E}_{\mathrm{line}}, \mathcal{R}_{\mathrm{line}})$. In this graph, each node $v \in \mathcal{V}_{\mathrm{line}}$ corresponds to an edge in the original graph $\mathcal{G}$. There will an edge $(u,v,r)$ between nodes $u, v \in \mathcal{V}_{\mathrm{line}}$ if the corresponding edges of $u$ and $v$ are adjacent in the original graph, and the edge type $r \in \{0, 1, \cdots, 7\}$ is determined by the angle $\angle_{(u,v)}$'s allocation in 8 equally-divided bins of $[0, \pi]$ ($\angle_{(u,v)}$ denotes the angle between the corresponding edges of $u$ and $v$ in the original graph). Based on this multi-relational line graph, we employ the GRMP layer (Sec.~\ref{sec:method:layer}) to propagate information between the nodes in $\mathcal{G}_{\mathrm{line}}$ and thus between the edges in the original graph $\mathcal{G}$. Readers are referred to \citet{zhang2022protein} for more details. We name the {\model} equipped with such an edge message passing scheme as {\model}-Edge. 

\textbf{Dataset details.} Two standard protein function prediction benchmarks are used in our experiments:
\vspace{-5mm}
\begin{itemize}
    \item \textbf{Enzyme Commission (EC) number prediction}~\cite{gligorijevic2021structure} requires the model to predict the EC numbers of a protein based on its tertiary structure, where the EC numbers describe the protein's catalysis of biochemical reactions. This task involves the binary prediction of 538 different EC numbers, forming 538 binary classification problems. This dataset contains 15,550 training, 1,729 validation and 1,919 test proteins.  
    \item \textbf{Gene Ontology (GO) term prediction}~\citep{gligorijevic2021structure} seeks to predict the GO terms owning by a protein based on its tertiary structure. This benchmark is further split into three branches based on three types of ontologies: biological process (BP), molecular function (MF) and cellular component (CC). Each branch is formed by multiple binary classification problems. The GO benchmark dataset contains 29,898 training, 3,322 validation and 3,415 test proteins. 
\end{itemize}
\vspace{-1mm}

\textbf{Model configurations.} The backbone architecture of {\model} is described in Sec.~\ref{sec:app:protein:arch}. Based on this backbone, we append a three-layer MLP with the architecture $\mathrm{Linear}(C_{\mathrm{out}}, C_{\mathrm{out}}) \rightarrow \mathrm{ReLU} \rightarrow \mathrm{Linear}(C_{\mathrm{out}}, C_{\mathrm{out}}) \rightarrow \mathrm{ReLU} \rightarrow \mathrm{Linear}(C_{\mathrm{out}}, N_{\mathrm{task}})$ to predict the binary classification logits of all tasks simultaneously ($C_{\mathrm{out}}$: the dimension of output protein representation; $N_{\mathrm{task}}$: the number of binary classification tasks). We employ the binary cross entropy loss for model optimization. 

\textbf{Training configurations.} An AdamW~\citep{loshchilov2017decoupled} optimizer (betas: [$0.9$, $0.999$], weight decay: $0$) is utilized to train the model for 200 epochs. We adopt a cosine learning rate scheduler to linearly increase the learning rate from $1.0 \times 10^{-7}$ to $1.0 \times 10^{-4}$, and the learning rate is decayed to $1.0 \times 10^{-6}$ in the rest epochs with a cosine rate. All models are trained with batch size 16 on 4 Tesla-V100-32GB GPUs (\emph{i.e.}, four proteins per GPU). 


\section{{\model} for Knowledge Graph Completion} \label{supp:sec:KG}


\begin{table}[h]
\begin{spacing}{1.1}
    \centering
    \caption{Performance comparison on knowledge graph completion benchmarks. ``$\downarrow$'' denotes the metric is the lower the better; ``$\uparrow$'' denotes the metric is the higher the better.}
    \label{supp:tab:kg_results}
    \vspace{0.5mm}
    \begin{adjustbox}{max width=1.0\textwidth}
        \begin{tabular}{ll|ccccc|ccccc}
            \toprule
            \multirow{2}{*}{\bf{Class}} & \multirow{2}{*}{\bf{Model}}
            & \multicolumn{5}{c|}{\bf{FB15k-237}} & \multicolumn{5}{c}{\bf{WN18RR}} \\
            & & \bf{MR}$_{\downarrow}$ & \bf{MRR}$_{\uparrow}$ & \bf{H@1}$_{\uparrow}$ & \bf{H@3}$_{\uparrow}$ & \bf{H@10}$_{\uparrow}$ & \bf{MR}$_{\downarrow}$ & \bf{MRR}$_{\uparrow}$ & \bf{H@1}$_{\uparrow}$ & \bf{H@3}$_{\uparrow}$ & \bf{H@10}$_{\uparrow}$ \\
            \midrule
            \multirow{4}{*}{\bf{Embedding}}
            & TransE & 357 & 0.294 & - & - & 0.465 & 3384 & 0.226 & - & - & 0.501 \\
            & DistMult & 254 & 0.241 & 0.155 & 0.263 & 0.419 & 5110 & 0.43 & 0.39 & 0.44 & 0.49 \\
            & ComplEx & 339 & 0.247 & 0.158 & 0.275 & 0.428 & 5261 & 0.44 & 0.41 & 0.46 & 0.51 \\
            & RotatE & 177 & 0.338 & 0.241 & 0.375 & 0.553 & 3340 & 0.476 & 0.428 & 0.492 & 0.571 \\
            \midrule
            \multirow{3}{*}{\bf{GNN}}
            & RGCN & 221 & 0.273 & 0.182 & 0.303 & 0.456 & 2719 & 0.402 & 0.345 & 0.437 & 0.494 \\
            & CompGCN & 197 & 0.355 & 0.264 & 0.390 & 0.535 & 3533 & 0.479 & 0.443 & 0.494 & 0.546 \\
            \rowcolor{Gray}
            \cellcolor{white} & {\model} & \bf{126} & \bf{0.374} & \bf{0.276} & \bf{0.415} & \bf{0.571} & \bf{680} & \bf{0.527} & \bf{0.472} & \bf{0.547} & \bf{0.636} \\ 
            \bottomrule
        \end{tabular}
    \end{adjustbox}
\end{spacing}
\end{table}


\textbf{Datasets.} We conduct experiments on two standard knowledge graphs, FB15k-237~\citep{toutanova2015observed} and WN18RR~\citep{dettmers2018convolutional}. FB15k-237 contains 14,541 entities, 237 relation, 272,115 training triplets, 17,535 validation triplets and 20,466 test triplets. WN18RR has 40,943 entities, 11 relations, 86,835 training triplets, 3,034 validation triplets and 3,134 test triplets. 
We follow the TorchDrug library~\citep{zhu2022torchdrug} to process knowledge graphs. For each triplet $<\!\!h,r,t\!\!>$, its flipped counterpart $<\!\!t,r^{-1},h\!\!>$ is included for data augmentation. All triplets from the validation and test sets are removed to form the fact graph for training. 

\textbf{Model architecture.} In this experiment, we instantiate the {\model} with 6 GRMP layers, each with 32 feature channels. Upon the {\model}, we adopt a two-layer MLP activated by ReLU to score each candidate triplet.  

\textbf{Training and evaluation.} For model training, we follow the default setting in the TorchDrug library~\citep{zhu2022torchdrug} to sample 32 negative triplets for each positive triplet and perform binary classification with the binary cross entropy loss. On both knowledge graphs, the {\model} is trained for 20 epochs by an Adam optimizer with learning rate $5.0 \times 10^{-3}$ and batch size 16. Model training is performed on 4 Tesla-V100-32GB GPUs. For evaluation, we follow previous works~\cite{vashishth2019composition,zhu2021neural} to report mean rank (MR), mean reciprocal rank (MRR) and HITS at N (H@N) for knowledge graph completion. 

\textbf{Baselines.} We compare the proposed {\model} with four classical knowledge graph embedding methods, \emph{i.e.}, TransE~\citep{bordes2013translating}, DistMult~\citep{yang2014embedding}, ComplEx~\citep{trouillon2016complex} and RotatE~\citep{sun2019rotate}, and two typical relational GNNs, \emph{i.e.}, RGCN~\citep{schlichtkrull2018modeling} and CompGCN~\citep{vashishth2019composition}. 

\textbf{Results.} We present the performance of {\model} and baselines in Tab.~\ref{supp:tab:kg_results}. It can be observed that {\model} clearly outperforms the embedding-based and GNN baselines on all metrics of two datasets. Although knowledge graphs contain no spatial information, they are representative multi-relational graphs and are good test fields for evaluating the capacity of relational GNNs. The superior performance of {\model} on these benchmarks demonstrates the effectiveness of the GRMP layer on modeling the complex relational patterns in knowledge graphs. 



\begin{wraptable}{r}{0.55\textwidth}
\vspace{-18.9mm}
\begin{spacing}{1.05}
\caption{\small Ablation study of the key components of GRMP on ImageNet-1K with {\model}-T.}
\label{supp:tab:ablation:grmp}
\begin{adjustbox}{max width=1.0\linewidth}
\begin{tabular}{l|ccc|c}
\toprule
\multirow{2}{*}{\bf{Setting}} & \bf{\#Params.} & \bf{FLOPs} & \bf{Throughput} & \bf{Top-1}  \\
& \bf{(M)} & \bf{(G)} & \bf{(imgs/s)} & \bf{Acc (\%)} \\
\midrule
GRMP & 28.8 & 4.6 & 530.3 & 82.3 \\
\midrule
GRMP (gating $\rightarrow$ addition) & 28.8 & 4.6 & 530.3 & 81.6$_{(\downarrow\!\;0.7)}$ \\
GRMP ($\alpha_r(v) \rightarrow |\mathcal{R}|^{-1}$) & 28.8 & 4.6 & 567.4 & 81.9$_{(\downarrow\!\;0.4)}$ \\
GRMP (\emph{w/o} $W^{\mathrm{in}}$) & 26.7 & 4.3 & 561.9 & 81.7$_{(\downarrow\!\;0.6)}$ \\
GRMP (\emph{w/o} $W^{\mathrm{out}}$) & 26.7 & 4.3 & 562.5 & 81.5$_{(\downarrow\!\;0.8)}$ \\
\bottomrule
\end{tabular}
\end{adjustbox}
\end{spacing}
\vspace{-5mm}
\end{wraptable} 


\section{More Ablation Study} \label{supp:sec:ablation}

\subsection{Effect of GRMP Components} \label{supp:sec:ablation:grmp}

In Tab.~\ref{supp:tab:ablation:grmp}, we analyze the key components of GRMP by substituting or removing the original component. This part of ablation studies are conducted on ImageNet-1K classification with {\model}-T.

\textbf{Effect of gating mechanism.} In the first row of the second block, we study the importance of the gating mechanism in GRMP by substituting the Hadamard product in the step {\small \circled{5}} of Eq.~\eqref{eq:grmp} with the addition. After such a change, the top-1 accuracy decays by 0.7\%. This performance decay demonstrates that, by using the separable graph convolution scheme in GRMP, the gating operation is more suitable than addition for node representation update (in contrast to the additive node representation update of RGConv in Eq.~\eqref{eq:rgconv}), which shares similar insights with the modulation mechanism in FocalNet~\citep{yang2022focal}.

\textbf{Effect of node-adaptive relation weighting.} In the second row of the second block, we replace GRMP's node-adaptive relation weighting operation with simply taking the mean over all relations. This change leads to a 0.4\% drop of accuracy. This relation weighting operation helps the GRMP layer to adaptively aggregate the messages of different relations based on each node's status, which benefits the model performance.

\textbf{Effect of pre-layer and post-layer node-wise channel aggregation.} In the third and fourth rows of the second block, we respectively evaluate the model variants without $W^{\mathrm{in}}$ and $W^{\mathrm{out}}$. Under these two settings, the model accuracy decays by 0.6\% and 0.8\%, respectively. Therefore, it is important to perform both pre-layer and post-layer node-wise channel aggregation in the GRMP layer.   


\subsection{Effect of Multi-range Edges for Protein Structure Modeling} \label{supp:sec:ablation:edge}


\begin{wraptable}{r}{0.36\textwidth}
\begin{spacing}{1.0}
\small
\vspace{-9.8mm}
\caption{\small Ablation study of multi-range edges on EC with {\model}.}
\label{supp:tab:ablation:edge}
\vspace{-1.8mm}
\center
\begin{adjustbox}{max width=1.0\linewidth}
\begin{tabular}{ccc|c}
\toprule
\bf{short} & \bf{medium} & \bf{long} & $\mathbf{F}_\mathbf{max}$ \\ 
\midrule
\checkmark & & & \bf{0.750} \\
& \checkmark & & 0.708 \\
& & \checkmark & 0.647 \\
\midrule
\checkmark & \checkmark &  & 0.755 \\
\checkmark & & \checkmark & \bf{0.760} \\
& \checkmark & \checkmark & 0.720 \\
\midrule
\checkmark & \checkmark & \checkmark & \bf{0.768} \\
\bottomrule
\end{tabular}
\end{adjustbox}
\end{spacing}
\vspace{-5.3mm}
\end{wraptable} 


Tab.~\ref{supp:tab:ablation:edge} shows the performance of {\model} on the EC function prediction benchmark by using different ranges of edges. When a single range of edges are employed, the model with short-range edges obtains the highest $\mathrm{F}_{\mathrm{max}}$ score 0.750. This result illustrates the importance of capturing short-range interactions for protein structure modeling, which coincides with the fact that many short-range interactions (\emph{e.g.}, peptide and hydrogen bonds) contribute to the formation of protein structures. By adding long-range edges, the model performance is improved to 0.760, where the extra modeling of long-range interactions (\emph{e.g.}, hydrophobic interactions) contributes to this improvement. By using all three ranges of edges, the full model of {\model} achieves the best $\mathrm{F}_{\mathrm{max}}$ score 0.768, which demonstrates the necessity of capturing short-, medium- and long-range interactions for protein structure modeling. 


\subsection{Effect of GRMP for Protein Structure Modeling} \label{supp:sec:ablation:layer}


\begin{wraptable}{r}{0.41\textwidth}
\begin{spacing}{1.035}
\vspace{-8.3mm}
\caption{\small Ablation study of multi-relational modeling layer on EC with {\model}.}
\label{supp:tab:ablation:layer}
\begin{adjustbox}{max width=1.0\linewidth}
\begin{tabular}{lc|c|c}
\toprule
\multirow{2}{*}{\bf{Layer}} & \bf{Hidden} & \bf{Throughput} & \multirow{2}{*}{$\mathbf{F}_\mathbf{max}$} \\
& \bf{Dimension} & \bf{(proteins/s)} & \\
\midrule
RGConv & 422 & 34.4 & 0.752 \\
\rowcolor{Gray}
GRMP & 512 & 34.6 & \bf{0.768} \\
RGConv & 512 & 31.2 & 0.767 \\
\rowcolor{Gray}
GRMP & 592 & 31.5 & \bf{0.780} \\
\bottomrule
\end{tabular}
\end{adjustbox}
\end{spacing}
\vspace{-5mm}
\end{wraptable} 


In Tab.~\ref{supp:tab:ablation:layer}, we compare between RGConv and GRMP under the comparable throughput (\emph{i.e.}, the number of proteins that the model can process in one second). All experiments are performed on EC with {\model}. (1) We first set the hidden dimension of GRMP as 512. Under the comparable throughput, RGConv can only have the dimension of 422, and its $\mathrm{F}_{\mathrm{max}}$ score 0.752 is lower than GRMP's 0.768. (2) We then increase RGConv's hidden dimension to 512. At this time, RGConv achieves the $\mathrm{F}_{\mathrm{max}}$ score 0.767 which is comparable to GRMP's performance under the same dimension, while its throughput is decreased by 3.2. Under the comparable throughput, GRMP can have the hidden dimension of 592, which leads to a higher $\mathrm{F}_{\mathrm{max}}$ score 0.780. These results demonstrate that GRMP owns a better efficiency-performance trade-off than RGConv on protein structure modeling. 


\begin{wraptable}{r}{0.56\textwidth}
\begin{spacing}{1.0}
\vspace{-17.1mm}
\caption{\small Sensitivity analysis of semantic neighbor size on ImageNet-1K with {\model}-T.}
\label{supp:tab:sensitivity:img_neighbor}
\vspace{0.4mm}
\begin{adjustbox}{max width=1.0\linewidth}
\begin{tabular}{c|cccccccc}
\toprule
\bf{\#Neighbors} & 3 & 6 & 9 & 12 & 15 & 18 & 21 & 24 \\
\midrule
\bf{Top-1 Acc (\%)} & 82.23 & 82.22 & 82.16 & \cellcolor{lightgray} 82.26 & 82.20 & \textbf{82.34} & 82.28 & \textbf{82.34} \\
\bottomrule
\end{tabular}
\end{adjustbox}
\end{spacing}
\vspace{-5mm}
\end{wraptable} 


\section{Sensitivity Analysis} \label{supp:sec:sensitivity}

\textbf{Image modeling sensitivity to semantic neighbor size.} In Tab.~\ref{supp:tab:sensitivity:img_neighbor}, we report the performance of {\model}-T on ImageNet-1K classification under different semantic neighbor sizes for medium-range edge construction. Though some marginal improvements are observed by using a larger neighborhood size (\emph{i.e.}, more than or equal to 18 neighbors), the image modeling performance on this task is in general insensitive to the semantic neighbor size. By default, {\model}-T uses 12 semantic neighbors (denoted by the gray cell in Tab.~\ref{supp:tab:sensitivity:img_neighbor}), which achieves comparable performance with the configurations using more semantic neighbors. 



\end{document}